\documentclass[journal]{IEEEtran}
\usepackage[pdftex]{graphicx}
\usepackage{booktabs}
\usepackage{subcaption}
\usepackage{algorithm}
\usepackage{hyperref}
\usepackage{algcompatible}
\hypersetup{
     colorlinks   = true,
     citecolor    = green
}
\usepackage{amsfonts}
\usepackage{dblfloatfix}
\usepackage{multirow}
\usepackage{multicol}
\usepackage{graphicx}   
\usepackage{makecell}
\usepackage{xcolor}

\renewcommand{\algorithmiccomment}[1]{\bgroup\hfill$\color{orange}\triangleright$~#1\egroup}
\newcommand{\RN}[1]{%
  \textup{\uppercase\expandafter{\romannumeral#1}}%
}
\usepackage{multirow}
\usepackage{amsmath}
\usepackage{amsthm}
\usepackage{cite}
\newtheorem{assumption}{Assumption}
\newtheorem{lemma}{Lemma}
\newtheorem{theorem}{Theorem}
\newtheorem{corollary}{Corollary}
\newcommand{\vect}[1]{\boldsymbol{#1}}

\begin{document}

\title{GraphHop: An Enhanced Label Propagation Method for Node Classification}
\author{Tian~Xie, Bin~Wang,~\IEEEmembership{Student~Member,~IEEE} and C.-C.~Jay~Kuo,~
\IEEEmembership{Fellow,~IEEE}%
\thanks{Tian~Xie, Bin~Wang and C.-C.~Jay~Kuo are with Ming Hsieh
Department of Electrical and Computer Engineering, University of
Southern California, Los Angeles, CA 90089, USA, e-mails: xiet@usc.edu
(Tian Xie), bwang28c@gmail.com (Bin Wang) and cckuo@ee.usc.edu (C.-C. Jay Kuo).}
}%

\maketitle

\begin{abstract} 

A scalable semi-supervised node classification method on
graph-structured data, called GraphHop, is proposed in this work.  The
graph contains attributes of all nodes but labels of a few nodes.  The
classical label propagation (LP) method and the emerging graph
convolutional network (GCN) are two popular semi-supervised solutions to
this problem. The LP method is not effective in modeling node attributes
and labels jointly or facing a slow convergence rate on large-scale graphs. GraphHop is proposed to its shortcoming. With proper
initial label vector embeddings, each iteration of GraphHop contains two
steps: 1) label aggregation and 2) label update.  In Step 1, each node
aggregates its neighbors' label vectors obtained in the previous
iteration. In Step 2, a new label vector is predicted for each node
based on the label of the node itself and the aggregated label
information obtained in Step 1.  This iterative procedure exploits the
neighborhood information and enables GraphHop to perform well in an
extremely small label rate setting and scale well for very large graphs.
Experimental results show that GraphHop outperforms state-of-the-art
graph learning methods on a wide range of tasks (e.g., multi-label and
multi-class classification on citation networks, social graphs, and
commodity consumption graphs) in graphs of various sizes. Our codes are
publicly available on GitHub \footnote{https://github.com/TianXieUSC/GraphHop}. 

\end{abstract}


\IEEEpeerreviewmaketitle

\section{Introduction}\label{sec:introduction}

\IEEEPARstart{T}{he} success of deep learning and neural networks often
comes at the price of a large number of labeled data. Semi-supervised
learning is an important paradigm that leverages a large number of
unlabeled data to address this limitation. The need for semi-supervised
learning has arisen in many machine learning problems and found wide
applications in computer vision, natural language processing, and
graph-based modeling, where getting labeled data is expensive and there
exists a large amount of unlabeled data. 

Among semi-supervised graph learning methods, label propagation (LP)
\cite{zhou2003learning, chapelle2006label} has demonstrated good
adaptability, scalability, and efficiency for node classification. Fig.
\ref{fig:label_propagation} gives a toy example of LP on a graph, where
data entities are represented by nodes and edges are either formed with
pairwise similarities of linked nodes or generated by a certain
relationship (e.g.,  social networks).  LP exploits the geometry of data
entities induced by labeled and unlabeled examples.  With fewer labeled
nodes, LP iteratively aggregates label embeddings from neighbors and
propagates them throughout the graph to provide labels for all nodes.
Typically, LP-based techniques have a small memory requirement
\cite{ravi2016large} and a fast convergence rate.  On the other hand,
LP-based methods are limited in their capability of integrating multiple
data modalities for effective learning. Nodes in a graph may contain
attributes and class labels. Only propagating the label information may
not be adequate.  Take the node classification problem for citation
networks as an example.  Each paper is represented as a node with a
bag-of-words attribute vector that describes its keywords and associated
class label. Edges in the graph are formed by citation links between
papers.  The three pieces of information (i.e., node attributes, node
labels, and edges) should be jointly considered to make accurate node
classification. Furthermore, the rules of label propagation in LP
methods are often implemented by a simple lookup table for unlabeled
nodes. They fail to leverage the label information in a probabilistic
way. 

%%%%%%%%%%%%%%%%%%%%%%%%%%%%%%%%%%%%%%%%%%%%%%%%%%%%%%%%%%%%%%%%%%%%%%%%%%%%%%%%
\begin{figure}[!t]
\centering
\includegraphics[width=0.5\textwidth]{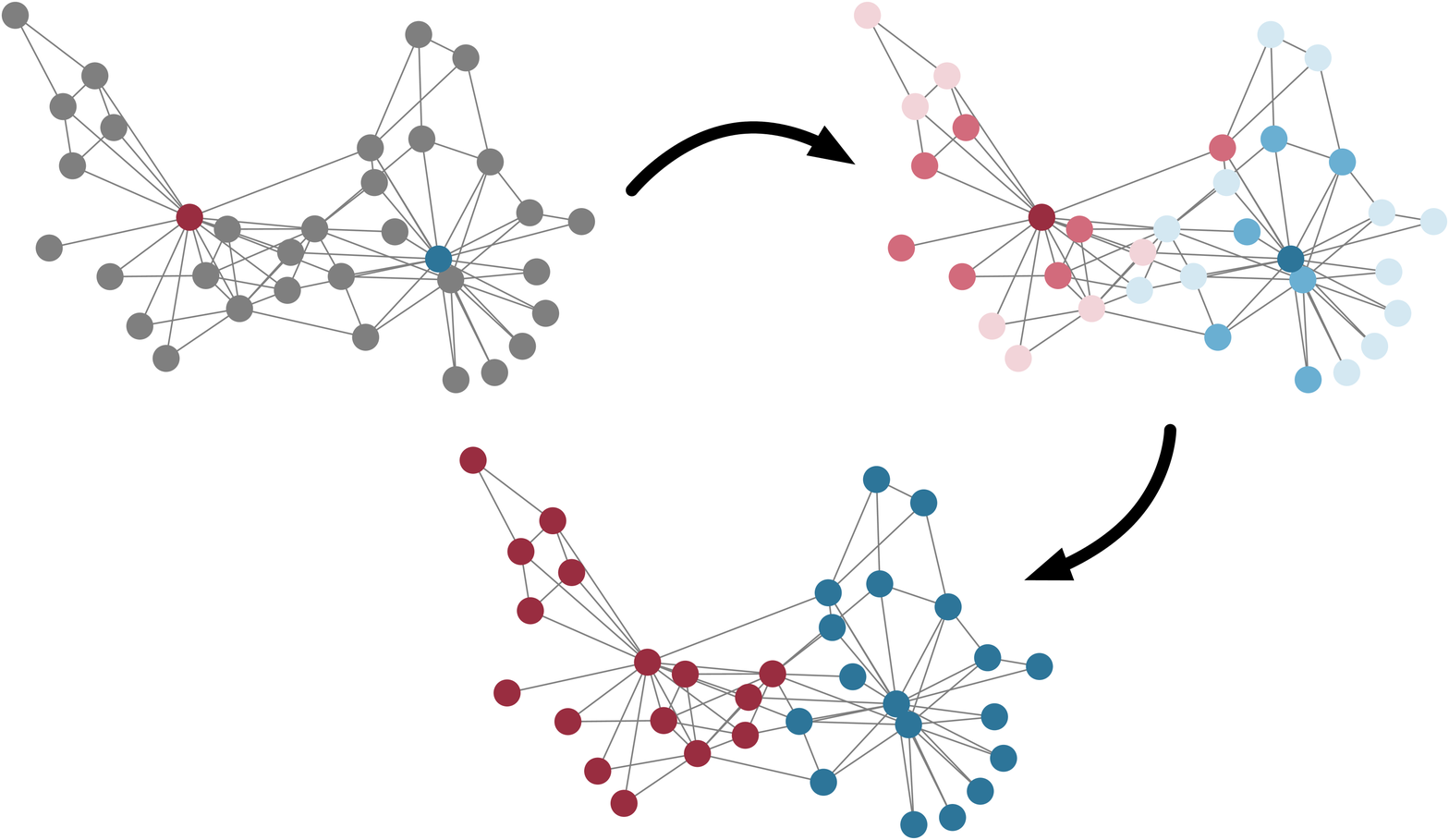}
\caption{A toy example of label propagation (LP) on graphs, where the dark
blue and red nodes represent different labeled examples and the grey
nodes denote unlabeled examples. The color saturation degree of a node
indicates its classification confidence level.}\label{fig:label_propagation}
\end{figure}
%%%%%%%%%%%%%%%%%%%%%%%%%%%%%%%%%%%%%%%%%%%%%%%%%%%%%%%%%%%%%%%%%%%%%%%%%%%%%%%%

Due to the recent success of neural networks \cite{lecun2015deep}, there
has been an effort of applying neural networks into graph-structured
data.  One pioneering technique, known as graph convolutional networks
(GCNs) \cite{kipf2016semi}, has achieved impressive node classification
performance for citation networks.  Instead of propagating label
embeddings as done by LP, GCNs conduct attributes propagation through a
graph with labels as supervision.  On one hand, GCNs incorporate node
attributes, node labels, and edges in model learning. On the other hand,
GCNs fail to exploit the label distribution in the graph structure.
Reliance on labels as supervision limits label update efficiency and
hinders the ability in an extremely small label rate scenario, e.g.
only one labeled example per class. Li, {\em et al.} \cite{li2019label}
proposed a new method called label efficient GCNs.  However, their
method does not induce unlabeled data but only changes the convolutional
filters in model learning. Others \cite{ma2019flexible, qu2019gmnn,
zhang2019bayesian} tried to incorporate label distributions in a
generative model based on special inductive bias assumptions
\cite{chapelle2009semi}. Furthermore, end-to-end training makes
neural-network-based solutions difficult to scale for large graphs.  The
rapid expansion of neighborhood sizes along deeper layers restricts GCNs
from batch training \cite{chiang2019cluster}. There is a trade-off
between training efficiency and the receptive field size
\cite{wu2019simplifying}. 

To address the weaknesses of LP and GCNs, a scalable semi-supervised
graph learning method, called GraphHop, is proposed in this work.
GraphHop fully exploits node attributes, labels, and link structures for
graph-structured data learning. Node attributes and labels are viewed as
signals, which are assumed to be locally smooth on graphs.  Attributes
exist in all nodes while labels are only available in a small number of
nodes (i.e., a typical semi-supervised setting).  GraphHop integrates
two signal types based on an iterative learning process.  Node
attributes are used to predict the label embedding vector at each node
in the initialization stage.  The dimension of the label embedding
vector is the same as the class number, where each element indicates the
probability of belonging to a class. After initialization, each
iteration of GraphHop contains two steps: 1) label aggregation and 2)
label update.  In Step 1, each node aggregates its neighbors' label
embedding vectors obtained in the previous iteration. In Step 2, a new
label embedding vector is predicted for each node based on the label
vector of the node itself and the aggregated label vector obtained in
Step 1.  The iterative procedure exploits multi-hop neighborhood
information, leading to the name of ``GraphHop", for efficient learning.
It enables GraphHop to perform well in an extremely small label rate setting and scale well for very large graphs. We conduct extensive experiments to validate the effectiveness of the proposed GraphHop method. We test a variety of semi-supervised node classification tasks including multi-class classification on various sizes of graphs and multi-label classification on protein networks. The proposed GraphHop method achieves superiorly in terms of prediction accuracy and training efficiency.

The rest of this paper is organized as follows. Some preliminaries are
introduced in Sec.  \ref{sec:preliminaries}. The GraphHop method is
presented in Sec.  \ref{sec:method}. The adoption of batch training and
the small number of parameters makes GraphHop efficient and scalable to
large-scale graphs.  Theoretical analysis is conducted to estimate the
required iteration number of the iterative algorithm in Sec.
\ref{sec:analysis}.  Extensive experiments are conducted to show the
state-of-the-art performance with low memory usage in Sec.
\ref{sec:experiments}.  Comments on related work are made in Sec.
\ref{sec:review}. Finally, concluding remarks are given and future
research directions are pointed out in Sec.  \ref{sec:conclusion}. 

\section{Preliminaries}\label{sec:preliminaries}

In this section, we first formulate the transductive semi-supervised
learning problem on graphs in Sec.  \ref{subsec:problem}.  Then, we
present the smoothness assumption, which holds for most graph
learning methods and provides its empirical evidence in Sec.
\ref{subsec:assumption}.  Finally, we review two commonly used graph
signal processing tools in Sec. \ref{subsec:GSP}. 

%%%%%%%%%%%%%%%%%%%%%%%%%%%%%%%%%%%%%%%%%%%%%%%%%%%%%%%%%%%%%%%
\begin{figure*}[ht]
\centering
\includegraphics[width=0.80\linewidth]{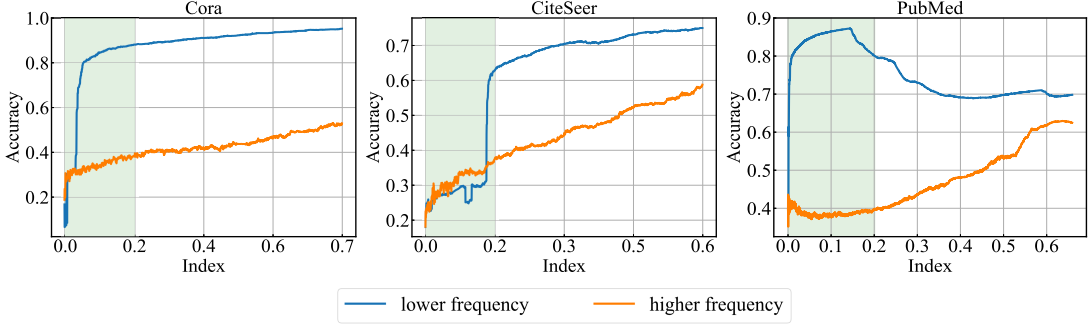}
\caption{The plot of accuracy curves based on lower frequency (in blue)
and higher frequency (in red) for citation datasets, respectively.  The
$x$-axis denotes the selected top-$k$ lower or higher frequency
components in percentage and the top 20\% low frequency region is shaded
in green.} \label{fig:label_signal}
\end{figure*}
%%%%%%%%%%%%%%%%%%%%%%%%%%%%%%%%%%%%%%%%%%%%%%%%%%%%%%%%%%%%%%%

\subsection{Problem Statement}\label{subsec:problem}

An undirected graph can be represented by a triple:
\begin{equation}\label{eq:graph_def}
\mathcal{G} = (\mathcal{V}, \vect{A}, \vect{X}),
\end{equation}
where $\mathcal{V}$ denotes a set of $n$ nodes, $\vect{A} \in
\mathbb{R}^{n \times n}$ is the adjacency matrix between nodes, and $\vect{X}
\in \mathbb{R}^{n\times d}$ is the attribute matrix whose
$d$-dimensional row vector is an attribute vector associated with each
node. In the setting of semi-supervised classification, each labeled
node belongs to one class, 
$$
y \in C, \mbox{   where  } C= \{1, \cdots, c\}.
$$
Nodes are divided into labeled and unlabeled sets with their indices denoted by
\begin{equation}\label{eq:labeled_nodes}
\mathcal{L} = \{1, \cdots, l\} \mbox{  and  } \mathcal{U} = 
\{l + 1, \cdots, n\},
\end{equation}
respectively. 

Let $\mathcal{H}$ be the set of $n\times c$ matrices with nonnegative
entries. Matrix 
$$
\vect{H}^T \in \mathcal{H} = (\vect{h}_1, \cdots, \vect{h}_l, \vect{h}_{l+1}, \cdots, \vect{h}_n)
$$ 
is the global label embedding matrix. It can be divided into 
$$
\vect{H}^T = (\vect{H}_l^T, \vect{H}_u^T)
$$
where $\vect{H}_l$ and $\vect{H}_u$ represent labeled and unlabeled examples,
respectively.  For labeled examples $i \in \mathcal{L}$, vector $\vect{h}_{i}$
is an one-hot vector, where $h_{ij} = 1$ if $x_i$ is labeled as $y_i =
j$ and $h_{ij} = 0$ otherwise. For unlabeled examples $i \in
\mathcal{U}$, vector $\vect{h}_{i}$ is a probability vector whose element
$h_{ij}$ is the predicted probability of belonging to class $j$.  To
classify unlabeled node $i$ in the convergence stage, we assign it the
label of the most likely class:
\begin{equation}\label{eq:node_classification}
y_i = \text{argmax}_{j\in{C}}h_{ij}. 
\end{equation}
The goal of transductive graph learning is to infer labels of unlabeled
nodes (i.e., matrix $\vect{H}_u$), given attribute matrix $\vect{X}$, adjacency matrix
$\vect{A}$, and labeled examples $\vect{H}_l$. In the extremely small label rate
case, the number of examples in the labeled set is much smaller than
that in the unlabeled set; namely,
$$
|\mathcal{L}| \ll |\mathcal{U}|. 
$$

\subsection{Assumption}\label{subsec:assumption}

Both the attribute matrix and the label embedding matrix can be treated
as signals on graphs. For example, the $d$-dimensional column vectors of
attribute matrix $\vect{X}$ can be viewed as a $d$-channel graph signal.  One
assumption for semi-supervised learning is that data points that are
close in a high density region should have similar attributes and
produce similar outputs. This property can be stated qualitatively
below. 
\begin{assumption}\label{assumption}
The node attribute and label signals are smooth functions 
on graphs.
\end{assumption}
The smoothness of attributes and labels can be verified by evaluating
their means and variances as a function of the length of the shortest
path between every two nodes, which is usually measured in terms of the
hop count. Here, we propose an alternative approach to verify the smooth
label assumption and conduct experiments on the Cora, CiteSeer and
PubMed datasets (see Sec.  \ref{sec:experiments} for more details) using
the following procedure. 

\noindent
{\bf Step 1.} Compute the graph Fourier basis 
$$
\vect{Q} = [\vect{q}_1, \vect{q}_2, \cdots, \vect{q}_n]
$$
and its corresponding frequency matrix 
$$
\vect{\Lambda} = \text{diag}(\lambda_1, \lambda_2, \cdots, \lambda_n)
$$ 
from graph Laplacian 
$$
\vect{L} = \vect{D} - \vect{A}, \mbox{  where  } \vect{D}=\text{diag}(a_{11}, a_{22}, \cdots, a_{nn})
$$ 
where frequencies in $\vect{\Lambda}$ are arranged in an ascending order:
\begin{equation}
\lambda_1 < \lambda_2 \leq \lambda_3 \leq \cdots \leq \lambda_n.
\end{equation}
Similarly, we can define the reverse-ordered frequency matrix,
$$
\vect{\Lambda}_r = \text{diag}(\lambda_n, \lambda_{n-1}, \cdots, \lambda_1),
$$
and find the corresponding basis matrix, 
$$
\vect{Q}_r = [\vect{q}_n, \vect{q}_{n-1}, \cdots, \vect{q}_1].
$$ 
It is well known that the basis signals associated with lower
frequencies (i.e., smaller eigenvalues) are smoother on the graph
\cite{zhu2009introduction}. 

\noindent
{\bf Step 2.} Compute the top-$k$ lowest or highest frequency components
of labels via
$$
\vect{\tilde{Y}} = \vect{Q}[:k]^T\vect{Y} \mbox{   or   } \vect{\tilde{Y}}_r = \vect{Q}_r[:k]^T\vect{Y},
$$
where each row of $Y$ denotes the one-hot embedding of a node label.
Then, we reconstruct labels for all nodes via 
$$
\vect{\hat{Y}} = \vect{Q}[:k]\vect{\tilde{Y}} \mbox{   or   } \vect{\hat{Y}}_r = \vect{Q}_r[:k]\vect{\tilde{Y}}_r.
$$
Make a prediction based on the reconstructed label embedding matrix, where
classes of node $i$ are decided by Eq. (\ref{eq:node_classification}).
The above computation yields two classification results for each node -
one using smooth components ($\vect{\hat{Y}}$) while the other using highly
fluctuating components ($\vect{\hat{Y}_r}$). 

We add the number of Laplacian frequency components incrementally,
reconstruct label embeddings for classification, and show two accuracy
curves for each dataset in Fig.  \ref{fig:label_signal}, where the one
using low frequency components is in blue and the one using high
frequency components is in red.  The blue curve rises up quickly while
the red curve increases slowly in all three cases, which support our assumption that label signals are smooth on graphs. Interestingly, the
performance does drop when adding more high frequency components for
PubMed. 

\subsection{Smoothening Operations}\label{subsec:GSP}

Based on the smoothness assumption of label signals, we expect that a
smoothening operation on node attributes and labels helps achieve better
classification results. There are several common ways to achieve the
smoothening effect, which can be categorized into the following three types. 
\begin{enumerate}
\item {\em Average attributes or label embedding vectors of neighboring
nodes.} It is well known that the average operation behaves like a
lowpass filter \cite{wu2019simplifying}. Another reason for taking the average is that nodes have
different numbers of $m$-hop neighbors, denoted by $\Omega_m$, where
$m=1,2,\cdots$. By averaging the attributes/labels of nodes with the
same hop distance, we can consider the impact of $m$-hop neighbors for
all nodes uniformly. 
\item {\em Use regression for label prediction.} We need to predict
labels for a great majority of nodes based on attributes of all nodes
and labels of a few nodes. One way for prediction is to train a
regressor. In the current setting, the logistic regression (LR)
classifier is a better choice due to the discrete nature of the output.
Regression (regardless of a linear or a logistic regressor) is
fundamentally a smoothening operation since it adopts a fixed but
smaller model size to fit a large number of observations. 
\item {\em Incorporate a smooth regularization term in an objective
function.} One can introduce a quadratic penalty term and, at the same
time, remain consistent with the initial labeling. This is used in the
label propagation (LP) method \cite{zhou2003learning, chapelle2006label}
as elaborated below. 
\end{enumerate}

The LP method can be formally defined as a minimization of cost function
\begin{equation}\label{equ:graph_regularizer}
J(\vect{H}) = \underbrace{||\vect{H}_l - \vect{Y}_l||_F^2}_{\text{Least square penalty}} +
\underbrace{\alpha \text{Tr}(\vect{H}^\top \vect{LH})}_{\text{Laplacian regularization}},
\end{equation}
where $\vect{Y}_l$ is the one-hot label embedding matrix of labeled nodes,
$||\cdot||_F$ is the Frobenius norm, $\alpha$ is a parameter controlling
the degree of the regularization, and $\vect{L}$ is the unnormalized graph
Laplacian. The optimization solution of embedding matrix $\vect{H}$ in Eq.
(\ref{equ:graph_regularizer}) can be derived equivalently by applying a
embedding propagation process, where the $t$th iteration update is
\begin{equation}\label{equ:pre_label_propagation}
\vect{H}^{(t + 1)} = \alpha \vect{D}^{-\frac{1}{2}}\vect{AD}^{-\frac{1}{2}}\vect{H}^{(t)} + (1 -
\alpha) \vect{H}^{(0)},
\end{equation}
where $\vect{H}^{(0)}_l = \vect{H}_l$ and zeros for unlabeled examples. The update
rule can be described by replacing the current label embeddings with
their averaged neighbors' through a look up table plus the initially
labeled examples. As shown in Eqs. (\ref{equ:graph_regularizer}) and (\ref{equ:pre_label_propagation}), we see that the LP method exploits
smoothening operations of the first and the third types, respectively. 

A closed-form solution to Eqs.  (\ref{equ:graph_regularizer}) and
(\ref{equ:pre_label_propagation}) can be equivalently cast in form of
\begin{equation}\label{equ:optimized_solution}
\vect{\hat{H}} = (1 - \alpha)(\vect{I} - \alpha \vect{L})^{-1} \vect{H}^{(0)}.
\end{equation}

Despite the simple and efficient propagation process in Eq.
(\ref{equ:pre_label_propagation}), the LP method is limited in its
learning abilities in label embedding vector update in the following
three aspects.
\begin{itemize}
\item Only label embeddings of immediate (one-hop) neighbors are being
propagated at each iteration according to Eq. (\ref{equ:pre_label_propagation}).  
It could be advantageous to consider label embeddings of nodes in a
larger neighborhood. 
\item The availability of node attributes is not leveraged.
\item The current label vectors of unlabeled nodes are fully replaced (Eq. (\ref{equ:pre_label_propagation})) by
propagated embeddings from neighbors of the previous iteration.  The
simple label update rule limits the ability to learn the correlation of
label embeddings between the node itself and its neighbors. 
\end{itemize}

%%%%%%%%%%%%%%%%%%%%%%%%%%%%%%%%%%%%%%%%%%%%%%%%%%%%%%%%%%%%%%%
\begin{figure*}[!t]
\centering
\includegraphics[width=\textwidth]{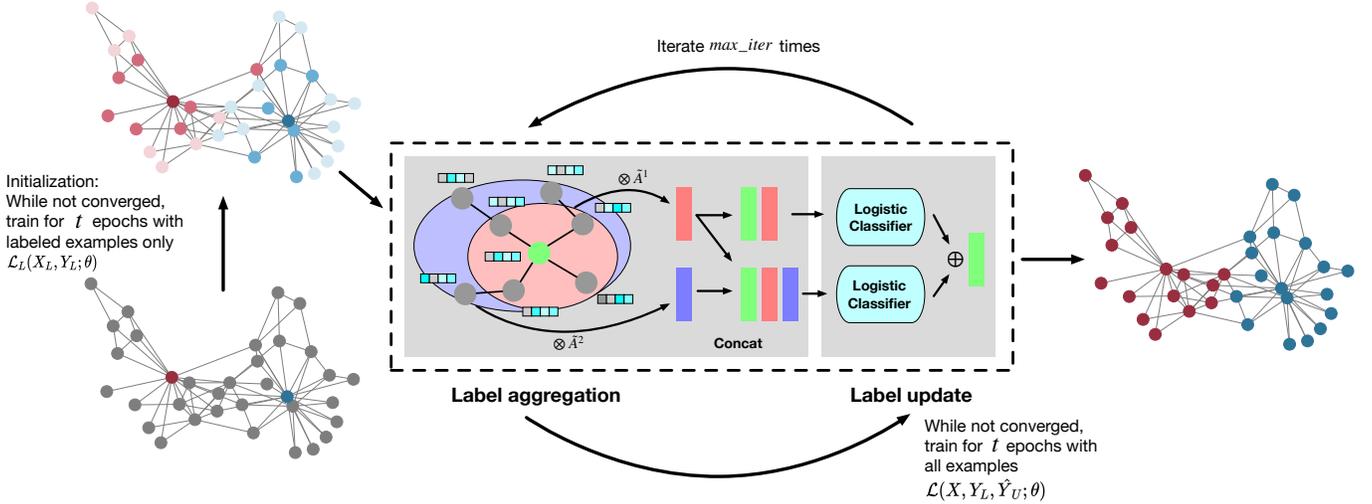}
\caption{Overview of the proposed GraphHop method, where the left
subfigure illustrates the initialization stage, the middle subfigure
shows the two steps (label propagation and label update) in the
iteration stage, and the converged result is given in the right
subfigure.} \label{fig:graphhop_model}
\end{figure*}
%%%%%%%%%%%%%%%%%%%%%%%%%%%%%%%%%%%%%%%%%%%%%%%%%%%%%%%%%%%%%%%

\section{GraphHop Method}\label{sec:method}

The GraphHop method is introduced in this section. It is an iterative
algorithm that contains an initialization stage and an iteration stage
as shown in Fig. \ref{fig:graphhop_model}. In the initialization stage,
it predicts the initial label embedding vector for all nodes based on
attribute vectors via regression. In the iteration stage, it focuses on
label propagation and update. No attributes information is needed in the
iteration stage and, as a result, its model size is small. In
particular, GraphHop has the following novel features to address the
limitations of the LP method as described in the last section. 
\begin{itemize}
\item It exploits both attribute and label embeddings of one- or
multi-hop neighbors. 
\item It leverages the information of node attributes in the
initialization stage. 
\item The label vectors of unlabeled nodes are predicted through
regression of (rather than being replaced by) label embeddings from one-
or multi-hop neighbors of the previous iteration. The regression-based
update rule, which is the smoothening operation of the second type as
mentioned above, boosts the learning power of GraphHop. 
\end{itemize}

\subsection{Initialization}\label{subsec:initialization}

The goal of the initialization stage is to offer good initialize label
embedding vectors for all nodes. Recall that the label embedding vector
is a $C$-dimensional vector whose elements indicate the probability of
a node belonging to class $c=1, \cdots, C$. If a node is a labeled one,
its initial embedding is a one-hot (or multi-hot in multi-label classification) vector. If a node has no label, we
can train an LR classifier that uses the attributes of the node itself
as well its multi-hop neighbors to predict the initial embeddings.  This
is justified by the underlying assumption of semi-supervised graph
learning; namely, nodes of spatial proximity should have similar
attributes and produce similar outputs (i.e.,  labels). 

The input to the LR classifier is the averaged attributes of
$m$-hops, $m=0,1,\cdots,M$, where $m=0$ denotes the self node. For node
$j$, we have the following averaged attribute vectors of dimension $d$:
\begin{equation}\label{equ:attri_aggregation_1}
\vect{x}_{j,m}= \frac{\sum_{l\in \Omega_m (j)} \vect{x}_l}{|\Omega_m (j)|}, 
\quad m=0,1,\cdots,M,
\end{equation}
where $\Omega_m (j)$ denotes the $m$-hop neighbors of node $j$.  This
yields an aggregated attribute matrix, $\vect{X}_M \in \mathbb{R}^{n \times
d(M+1)}$, whose $d(M+1)$-dimensional row vector is the concatenation of
these $(M+1)$ vectors associated with each node. Mathematically, we
can express it as
\begin{equation}\label{attri_aggregation_2}
\vect{X}_M = \|_{0\le m \le M}\; \vect{\tilde{A}}^{m} \vect{X},
\end{equation}
where $\|$ denotes column-wise concatenation, $\vect{\tilde{A}}^m$ is the
normalized $m$-hop adjacency matrix, and $m$ is the hop index.  The
output of the LR classifier is the initial label embeddings matrix,
denoted by $\vect{H}^{(0)}\in \mathbb{R}^{n \times C}$, of all nodes.  This
aggregation procedure is illustrated in Fig.
\ref{fig:detail_label_propagation}. 

We first use the labeled examples to determine the model parameters of the
LR classifier in the training phase. Afterwards, we use the trained
LR classifier to infer the initial label embeddings of unlabeled
nodes. Mathematically, this is expressed as
\begin{equation}\label{equ:initial_infer}
\vect{H}^{(0)} = \mathrm{p}_\text{model}(\vect{Y}|\vect{X}_M;\vect{\theta}),
\end{equation}
where $\mathrm{p}_\text{model}$ denotes the predictor based on the LR
classifier, $\vect{Y}$ is the label embeddings of known examples and $\vect{\theta}$
is the set of learned model parameters. 

%%%%%%%%%%%%%%%%%%%%%%%%%%%%%%%%%%%%%%%%%%%%%%%%%%%%%%%%%%%%%%%
\begin{figure}[ht]
\centering
\includegraphics[width=0.48\textwidth]{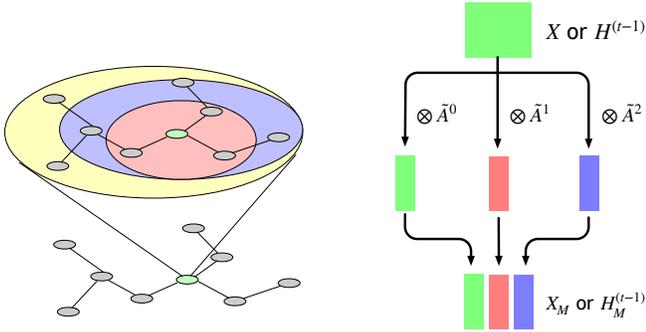}
\caption{Illustration of attribute and label aggregation in GraphHop,
where multi-hops attributes are first averaged and then 
concatenated to form input to the LR classifier.} 
\label{fig:detail_label_propagation}
\end{figure}
%%%%%%%%%%%%%%%%%%%%%%%%%%%%%%%%%%%%%%%%%%%%%%%%%%%%%%%%%%%%%%%

\subsection{Iteration}\label{subsec:iteration}

After initialization, the GraphHop method processes label embedding
vectors via a two-step procedure at each iteration as elaborated below. 

\subsubsection{Step 1: Label Aggregation}

The label aggregation mechanism is exactly the same as the attribute
aggregation mechanism as stated in Sec. \ref{subsec:initialization}.
By modifying Eq. (\ref{attri_aggregation_2}) slightly, we have the 
following label aggregation formula
\begin{equation}\label{label_aggregation}
\vect{H}_M^{(t - 1)} = \|_{0\le m \le M} \vect{\tilde{A}}^{m}\vect{H}^{(t-1)}, \quad t=1,2,\cdots.
\end{equation}
The label aggregation step with $M=2$ is shown in Figs.
\ref{fig:graphhop_model} and \ref{fig:detail_label_propagation} as an
example, where the green node is the center node, nodes in the red
region are one-hop neighbors and nodes in the blue region are two-hop
neighbors. The averaged label embeddings of one-hop neighbors and
two-hop neighbors yield two output label embedding vectors,
respectively. They are concatenated with the label embedding of the
center node, which serves as the input to the label update step.  As
shown in Fig. \ref{fig:graphhop_model}, we can form multiple aggregated outputs and for each used as input to label update step for ensembles. Yet, a larger $M$ value results in larger memory and computational
complexity.  Even with the one-hop connection (i.e., $M=1$), the
propagation region is still increasing due to iteration. This is not
available in the model of \cite{wu2019simplifying} (see Eq.
(\ref{equ:simplify_gcn})). In experiments, we focus on the case with
$M=2$.

\subsubsection{Step 2: Label Update}

A label update operation is conducted on the propagated label embeddings
to smoothen the label embeddings furthermore.  For traditional LP, the
update is formulated as a simple replacement as given in Eq.
(\ref{equ:pre_label_propagation}). This update rule does not explore
embedding correlations between nodes and their neighbors fully.  For
improvement, we propose to train the label update operation and view it
as a predictor from concatenated label embeddings $\vect{H}_M^{(t-1)}$ at
iteration $(t-1)$ to the new label embeddings at iteration $t$. 
Mathematically, we have
\begin{equation}\label{label_propagation}
\vect{H}^{(t)} = \mathrm{p}_\text{model}(\vect{Y}|\vect{H}_M^{(t-1)};\vect{\theta}).
\end{equation}
Here, we choose multiple LR classifiers (one for each $m$) as the
predictors and average the predicted label embeddings as shown in Fig.
\ref{fig:graphhop_model}. Note that there are many choices for the
predictor in Eq. (\ref{label_propagation}), say, \cite{ou2016asymmetric,
kingma2013auto, goodfellow2014generative}. The choice for the LR
classifier is because of its efficiency and smaller model size. 
The implementation details of the LR classifier are elaborated in
Appendix \ref{sec:LR}.

Finally, the GraphHop method is summarized by pseudo-codes in Algorithm
\ref{alg:graphhop}. 

%%%%%%%%%%%%%%%%%%%%%%%%%%%%%%%%%%%%%%%%%%%%%%%%%%%%%%%%%%%%%%
\begin{algorithm}
\caption{GraphHop}
\label{alg:graphhop}
\begin{algorithmic}[1]
    \STATE \textbf{Input}: Graph $\vect{A}$, attributes $\vect{X}$, label vectors $\vect{H}_l$
    \STATE \textbf{Output}: Label vectors $\vect{H}_u$

    \STATE \textbf{Initialization}:
    \STATE $\vect{X}_M\gets$ calculate Eq. (\ref{attri_aggregation_2})
    \WHILE{not converged}
        \FOR{each batch} \label{alg_line:batch_1}
            \STATE Compute $\vect{g}\gets\nabla\mathcal{L}_L(\vect{X}_L, \vect{Y}_L;\vect{\theta})$ in Eq. (\ref{equ:label_loss})
            \STATE Conduct Adam update using gradient estimator $\vect{g}$
        \ENDFOR
    \ENDWHILE
    \STATE $\vect{H}^0 \gets$ calculate Eq. (\ref{equ:initial_infer}) 
    \STATE \textbf{Iteration}:
    \FOR{iteration $\in [1, ..., max\_iter]$}
        \STATE $\vect{H}^{(t - 1)}_M\gets$ calculate Eq. (\ref{label_aggregation}) \COMMENT{{\color{orange} \textit{label aggregation}}}
        \WHILE{not converged}
            \FOR{each batch} \label{alg_line:batch_2}
                \STATE Compute $\vect{g}\gets\nabla\mathcal{L}(\vect{X}, \vect{H}; \vect{\theta})$ in Eq. (\ref{equ:final_loss})
                \STATE Conduct Adam update using gradient estimator $\vect{g}$
            \ENDFOR
        \ENDWHILE
        \STATE $\vect{H}^{(t)}\gets$ calculate Eq. (\ref{label_propagation}) \COMMENT{{\color{orange} \textit{label update}}}
    \ENDFOR
\end{algorithmic}
\end{algorithm}
%%%%%%%%%%%%%%%%%%%%%%%%%%%%%%%%%%%%%%%%%%%%%%%%%%%%%%%%%%%%%%

\section{Analysis}\label{sec:analysis}

\subsection{Convergence Analysis}\label{sec:convergence_analysis}

We analyze the relationship between the lower bound of iteration number,
label rate, and the hops number in the label propagation step.  The
higher-order (large hops) of nodes was efficiently approximated within
$k$-hop neighbors in \cite{zhang2018link}. Here, we call the
approximation of each unlabeled node is \em{sufficient} if at least one
initially labeled examples are captured in its $k$-hop surroundings
after iteration, and otherwise \em{insufficient}. If all unlabeled nodes
are sufficient, we call the iteration sufficient. 

We define the shortest distance between nodes $u$ and $v$ as $dist(u,
v)$ in the graph. Given initially labeled examples, we partition node
set $\mathcal{V}$ into several subsets
$$
\mathcal{V}_0, \cdots, \mathcal{V}_i, \cdots, \mathcal{V}_I,
$$ 
where $\mathcal{V}_0$ denotes the set of initially labeled nodes,
$$
\mathcal{V}_i = \{u|\text{min}_{v\in \mathcal{V}_l}dist(u, v) = i\}
$$
is the subset of unlabeled nodes that have the same minimum distance
(i.e., of $i$ hops) to one of labeled nodes, and $I$ is the maximum
distance between an unlabeled node and a labeled one.  Since each node
will aggregate $k$-hop neighbors embeddings at one iteration, we have 
the following Lemma. 
\begin{lemma}\label{lemma}
Given the maximum hop $k = \text{max}\{K\}$ covered in each label
propagation step, the node predictions in $\mathcal{V}_i$ will be
insufficient until the propagation of iteration $i / k$ is finished. 
\end{lemma}

Based on Lemma \ref{lemma}, we can derive the number of sufficient nodes
after $t$ iterations.
\begin{theorem}\label{theorem}
Let $G$ be a graph with $n$ nodes, and the maximum degree of all nodes
in $d (d\neq 1)$. The number of initially labeled nodes is $j$. Then after
$t$ iterations, at most $\text{min}\{n, jd\frac{d^{kt} - 1}{d - 1}\}$
nodes are sufficient. 
\end{theorem}

\begin{proof}
Since $|\mathcal{L}| \ll |\mathcal{U}|$ in semi-supervised learning, we ignore the
difference between the number of unlabeled nodes (i.e. $|\mathcal{U}|$)and all
nodes in the graph (i.e. $|\mathcal{L}|+|\mathcal{U}|$). According to Lemma \ref{lemma},
after $t$ iterations, the nodes in subsets $\mathcal{V}_1$,
$\mathcal{V}_2$, $\cdots$, $\mathcal{V}_{kt}$ are sufficient. Then, we
have
\begin{equation}\label{equ:theorem}
\begin{aligned}
  |\mathcal{V}_1\cup\mathcal{V}_2\cup \cdots \cup\mathcal{V}_{kt}| & 
  = |\mathcal{V}_1| + |\mathcal{V}_2| + \cdots + |\mathcal{V}_{kt}|\\
  & \leq jd + jd^2 + ... + jd^{kt} \\
  & = jd\frac{d^{kt} - 1}{d - 1}
\end{aligned}
\end{equation}
The first equality is because that $\mathcal{V}_1, \mathcal{V}_2, ...,
\mathcal{V}_{kt}$ are mutually disjoint. Each node has only one unique
minimum distance to labeled set $\mathcal{V}_0$ so that they can only be
assigned to one specific subset. The inequality is due to the use of the
maximum degree $d$ for every node.  Apparently,
$|\mathcal{V}_1\cup\mathcal{V}_2\cup \cdots \cup\mathcal{V}_{kt}|\leq n$.
Thus, we get
\begin{equation}
|\mathcal{V}_1\cup\mathcal{V}_2\cup ... \cup\mathcal{V}_{kt}| \leq
\text{min}\{n, jd\frac{d^{kt} - 1}{d - 1}\},
\end{equation}
and the theorem is proved.
\end{proof}

It is easy to get the following corollary.
\begin{corollary}\label{corollary}
The predictions of all unlabeled nodes on graph $G$ will be sufficient
with $t$ iterations, where
\begin{equation}
t \in \Omega(\frac{1}{k}\text{log}_d(1 + \frac{n(d - 1)}{jd})).
\end{equation}
\end{corollary}
\begin{proof}
According to Theorem \ref{theorem}, at most $\text{min}\{n,
jd\frac{d^{kt} - 1}{d - 1}\}$ nodes are sufficient after $t$ iterations.
To ensure that all nodes on graph $G$ are sufficient, we let
$$
jd\frac{d^{kt} - 1}{d - 1} \ge n
$$ 
so that
\begin{equation}\label{equ:iteration_hops}
t \ge \frac{1}{k}\text{log}_d(1 + \frac{n(d - 1)}{jd}).
\end{equation}
\end{proof}

The result in Corollary \ref{corollary} shows that the relationship
between sufficient iterations and the maximum hop number in $K$ is in
inverse ratio. Increasing $k$ will decrease the required number of
iterations. The initial label rate $j$ and graph density $d$ also have
an influence on the iteration number. Yet, the effects are minor due to the
logarithm function.  Particularly, in a large-scale graph where $j \ll n$,
changing the label rate has negligible influence on the iteration
number. The same behavior has been shown in the experiment section. In
practice, we observe that GraphHop converges in few iterations 
(usually 10) since few iterations are required to achieve sufficiency. 

\subsection{Complexity Analysis}\label{subsec:complexity_analysis}

The time and memory complexities of GraphHop are significantly lower
that the traditional training in graph neural networks. The reason is
that GraphHop can do minibatch training, which is not available
GCN-based methods due to the neighbor expansion problem
\cite{chiang2019cluster}. The LR classifiers are directly trained based
on the output of Eq.  (\ref{label_propagation}), where minibatches can be easily
applied to matrix $\vect{H}_M^{(t - 1)}$ (see lines
\ref{alg_line:batch_1} and \ref{alg_line:batch_2} in Algorithm
\ref{alg:graphhop}) as input.  

Suppose that the total training minibatch number is $N$, the number of
iterations is $t$, the minibatch size is $b$, and the number of classes is $c$. Then, the time complexity 
of one minibatch propagation is 
$$
O(t\frac{||\vect{\tilde{A}}||_0}{N}c + tbc^2),
$$
where the first term comes from label aggregation while the second from 
label update. The memory usage complexity is 
$$
O(bc + c^2),
$$ 
which represents one minibatch embeddings and parameters for the
classifiers. Note that the parameter size is fixed and independent of
iterations $t$ and scales linearly in terms of the minibatch size $b$. 

\section{Experiments}\label{sec:experiments}

We conduct experiments to evaluate the performance of GraphHop with
multiple datasets and tasks. Datasets used in the experiments are
described in Sec.  \ref{subsec:datasets}. Experimental settings are
discussed in Sec. \ref{subsec:setting}. Then, the performance of
GraphHop is compared with state-of-the-art methods in small- and
large-scale graphs in Sec. \ref{subsec:benchmark}.  Finally, ablation
studies are given in Sec.  \ref{subsec:ablation_study}. 

\subsection{Datasets}\label{subsec:datasets}

We evaluate the performance of GraphHop on six representative graph
datasets as shown in Table \ref{tab:datasets}. Cora, CiteSeer, and
PubMed \cite{kipf2016semi} are three citation networks. Nodes are papers
while edges are citation links in these graphs. The task is to predict
the category of each paper.  PPI and Reddit \cite{hamilton2017inductive}
are two datasets of large-scale networks. PPI is a multi-label dataset,
where each node denotes one protein with multiple labels in the gene
ontology sets (121 in total).  Amazon2M \cite{chiang2019cluster} is by
far the largest graph dataset that is publicly available with over 2 millions nodes and 61
millions of edges obtained from Amazon co-purchasing networks. The raw
node features are bag-of-words extracted from product
descriptions. We use the principal component analysis (PCA)
\cite{hotelling1933analysis} to reduce their dimension to $100$. Also,
we use the top-level category as the class label for each node. 

%%%%%%%%%%%%%%%%%%%%%%%%%%%%%%%%%%%%%%%%%%%%%%%%%%%%%%%%%%%%%%
\begin{table}[ht]
\renewcommand{\arraystretch}{1.3}
\caption{Statistics of six representative datasets used in experiments.}
\label{tab:datasets}
\centering
\begin{tabular}{ccccc} \hline
    Dataset & Vertices & Edges & Classes & Features Dims \\ \hline
    Cora & $2,708$ & $5,429$ & $7$ & $1433$ \\
    CiteSeer & $3,327$ & $4,732$ & $6$ & $3703$ \\
    PubMed & $19,717$ & $44,338$ & $3$ & $500$ \\
    PPI & $56,944$ & $1,612,348$ & $121$ & $50$ \\
    Reddit & $231,443$ & $11,606,919$ & $41$ & $602$ \\
    Amazon2M & $2,449,029$ & $61,859,140$ & $47$ & $100$ \\ \hline
\end{tabular}
\end{table}
%%%%%%%%%%%%%%%%%%%%%%%%%%%%%%%%%%%%%%%%%%%%%%%%%%%%%%%%%%%%%%

%%%%%%%%%%%%%%%%%%%%%%%%%%%%%%%%%%%%%%%%%%%%%%%%%%%%%%%%%%%%%%
\begin{table}[ht]
\renewcommand{\arraystretch}{1.3}
\caption{The training, validation, and testing splits used in the
experiments, where the node numbers and the corresponding percentages
(in brackets) are listed.} \label{tab:data_split}
\centering
\resizebox{0.48\textwidth}{!}{\begin{tabular}{c|cccc|c} \hline
    \multirow{2}{*}{Dataset} & \multicolumn{5}{c}{Data splits} \\
    \cline{2-6}
    & \multicolumn{4}{c|}{Train} & Validation \\ \hline
    PPI & $569 (1\%)$ & $1139 (2\%)$ & $2847 (5\%)$ & $5694 (10\%)$ & $5000$ \\ \hline
    Reddit & $2296 (1\%)$ & $4592 (2\%)$ & $11562 (5\%)$ & $22888 (10\%)$ & $5000$\\ \hline
    Amazon2M & $19606 (1\%)$ & $35259 (2\%)$ & $77700 (3\%)$ & $139556 (5\%)$ & $50000$ \\ \hline
\end{tabular}}
\end{table}
%%%%%%%%%%%%%%%%%%%%%%%%%%%%%%%%%%%%%%%%%%%%%%%%%%%%%%%%%%%%%%

%%%%%%%%%%%%%%%%%%%%%%%%%%%%%%%%%%%%%%%%%%%%%%%%%%%%%%%%%%%%%%
\begin{table}[ht]
\renewcommand{\arraystretch}{1.3}
\caption{Grid search ranges of the hyperparameters.}\label{tab:parameters_tuning}
\centering
\begin{tabular}{c|c} \hline
$T$ & $\{0.1, 1, 10, 100\}$\\ \hline
$\alpha$ & $\{0.01, 0.1, 1, 10, 100\}$ \\ \hline
$\beta$ & $\{0, 0.1, 1, 10\}$\\ \hline
\end{tabular}
\end{table}
%%%%%%%%%%%%%%%%%%%%%%%%%%%%%%%%%%%%%%%%%%%%%%%%%%%%%%%%%%%%%%

%%%%%%%%%%%%%%%%%%%%%%%%%%%%%%%%%%%%%%%%%%%%%%%%%%%%%%%%%%%%%%
\begin{table}[ht]
\renewcommand{\arraystretch}{1.3}
\caption{Hyperparameters used in the experiments for the largest label rate.}\label{tab:parameters}
\centering
\begin{tabular}{c|ccc}    \hline
    Datasets & $T$ & $\alpha$ & $\beta$ \\ \hline \hline
    Cora & $0.1$ & $10$ & $1$ \\     \hline
    CiteSeer & $0.1$ & $10$ & $1$ \\ \hline
    PubMed & $0.1$ & $1$ & $1$\\     \hline
    Reddit & $1$ & $1$ & $0$ \\     \hline
    PPI & $1$ & $1$ & $1$ \\     \hline
    Amazon2M & $1$ & $100$ & $100$ \\ \hline
\end{tabular}
\end{table}
%%%%%%%%%%%%%%%%%%%%%%%%%%%%%%%%%%%%%%%%%%%%%%%%%%%%%%%%%%%%%%

\subsection{Experimental Settings}\label{subsec:setting}

We evaluate GraphHop and several benchmarking methods on the
semi-supervised node classification task in a transductive setting at a
number of label rates. For citation datasets, we first conduct
experiments by following the conventional train/validate/test split
(i.e., 20 labels per class) of the training set.  Next, we train models
at extremely low label rates (i.e., 1, 2, 4, 8, and 16 labeled
examples per class).  For the PPI, Reddit, and Amazon2M three
large-scale networks, the original data splits target at inductive
learning scenario, which do not fit our purpose.  To tailor them to the
transductive semi-supervised setting, we adopt fewer labeled training
examples. To be specific, for Reddit and Amazon2M, we pick the same
number of examples in each class randomly with multiple label rates for
training. For the multi-label PPI dataset, we simply select a small
portion of examples randomly in the training. A fixed size of remaining
examples is selected as validation while the rest is used as testing.
The full data splits are summarized in Table \ref{tab:data_split}. For
simplicity, we use the percentages of training examples to indicate
different data split in reporting performance results. 

We implement GraphHop in PyTorch \cite{paszke2017automatic}. For the LR
classifiers in the initialization stage and the iteration stage, we use
the same Adam optimizer with a learning rate of $0.01$ and
$5\times10^{-5}$ weight decay. The batch size is fixed to $512$ for
citation networks but adaptive for large-scale graphs since experiments
show there is a trade-off using different batch sizes for the latter
case. The training epochs are set to $1000$ with an early stopping
criterion which stops the classifiers from training and goes to the next
iteration. We set the maximum iteration to $100$ for citation datasets
and $200$ for large-scale networks. These numbers are large enough for
GraphHop to converge as shown in the experiments. For hyperparameters
$T$, $\alpha$, and $\beta$, we perform a grid search in the parameter
space based on the validation results.  The hyperparameter tuning
ranges and their final values are listed in Table
\ref{tab:parameters_tuning} and \ref{tab:parameters}, respectively. Note
that hyperparameters are tuned for different label rates.  We show
their values with respect to the largest label rates in Table
\ref{tab:parameters}.  All experiments were conducted on a machine with
a NVIDIA Tesla P100 GPU (16 GB memory), 10-core Intel Xeon CPU ($2.40$
GHz), and $100$ GB of RAM. 

\subsection{Performance Evaluation}\label{subsec:benchmark}

%%%%%%%%%%%%%%%%%%%%%%%%%%%%%%%%%%%%%%%%%%%%%%%%%%%%%%%%%%%%%%
\begin{table*}[ht]
\renewcommand{\arraystretch}{1.3}
\caption{Classification accuracy (\%) for three citation datasets 
with different label rates. The highest accuracy in each
column is highlighted in \textbf{bold} and the top three are \underline{underlined}. }\label{tab:results_citation}
\centering
\setlength{\tabcolsep}{1.5mm}{
\begin{tabular}{c|cccccc|cccccc|cccccc}
    \hline
    & \multicolumn{6}{c|}{\textbf{Cora}} & \multicolumn{6}{c|}{\textbf{Citeseer}} & \multicolumn{6}{c}{\textbf{Pubmed}} \\
    \hline
    \# of labels per class & 1 & 2 & 4 & 8 & 16 & 20 & 1 & 2 & 4 & 8 & 16 & 20 & 1 & 2 & 4 & 8 & 16 & 20 \\
    \hline
    \hline
    \textbf{LP} & $51.5$ & $56.0$ & $61.5$ & $63.4$ & $65.8$ & $67.3$ &    $30.1$ & $33.6$ & $38.2$ & $40.6$ & $43.4$ & $44.8$ &    $\underline{55.7}$ & $58.8$ & $62.7$ & $64.4$ & $65.8$ & $66.4$  \\
    \textbf{DeepWalk} & $40.4$ & $47.1$ & $56.6$ & $62.4$ & $67.7$ & $69.9$ & $28.3$ & $31.5$ & $36.4$ & $40.1$ & $43.8$ & $45.6$ & - & - & - & - & - & - \\
    \textbf{LINE} & $49.4$ & $56.0$ & $63.0$ & $67.3$ & $72.6$ & $74.0$ &   $28.0$ & $31.4$ & $36.4$ & $40.6$ & $45.8$ & $48.5$ & - & - & - & - & - & - \\
    \textbf{GCN} & $42.4$ & $52.0$ & $65.0$ & $72.2$ & $\underline{78.4}$ & $\underline{80.2}$ &  $36.4$ & $43.4$ & $53.9$ & $60.4$ & $\underline{67.5}$ & $68.8$ &   $41.3$ & $48.1$ & $59.3$ & $67.4$ & $74.5$ & $\underline{77.8}$ \\
    \textbf{GAT} & $41.8$ & $51.8$ & $66.4$ & $73.6$ & $77.8$ & $\underline{79.6}$ &   $32.8$ & $40.7$ & $51.8$ & $57.9$ & $64.5$ & $68.2$ & - & - & - & - & - & - \\
    \textbf{DGI} & $\underline{55.3}$ & $\textbf{63.1}$ & $\textbf{71.8}$ & $\underline{74.5}$ & $77.2$ & $77.9$ &   $\underline{46.1}$ & $\underline{52.7}$ & $\textbf{61.7}$ & $\textbf{65.6}$ & $\underline{68.2}$ & $\underline{68.7}$ & - & - & - & - & - & - \\
    \textbf{Graph2Gauss} & $\underline{54.5}$ & $\underline{61.3}$ & $\underline{69.5}$ & $72.4$ & $74.8$ & $75.8$ &   $\underline{45.1}$ & $\underline{50.8}$ & $\underline{58.4}$ & $\underline{61.7}$ & $64.4$ & $65.7$ & - & - & - & - & - & - \\
    \textbf{Co-training GCN} & $53.1$ & $\underline{59.4}$ & $68.0$ & $73.5$ & $\underline{78.9}$ & $78.7$ &    $36.7$ & $42.9$ & $52.0$ & $57.9$ & $62.5$ & $65.9$ &    $55.1$ & $\underline{59.9}$ & $\underline{66.9}$ & $\underline{71.3}$ & $\textbf{75.7}$ & $\textbf{77.9}$ \\
    \textbf{Self-training GCN} & $40.6$ & $52.3$ & $67.5$ & $\underline{73.8}$ & $77.3$ & $79.1$ &  $34.6$ & $42.3$ & $54.4$ & $\underline{63.1}$ & $\textbf{68.3}$ & $\underline{69.1}$ &  $49.7$ & $56.2$ & $65.0$ & $68.9$ & $73.6$ & $76.5$ \\
    % \textbf{IGCN} & - & - & $\underline{70.3}$ & - & - &    -  & - & $\underline{58.0}$ & & &       - & - & $70.1$ & - & -  \\ 
    \hline
    \hline
    \textbf{GraphHop} & $\textbf{59.8}$ & $58.2$ & $\underline{69.3}$ & $\textbf{76.3}$ & $\textbf{79.7}$ & $\textbf{81.0}$ & $\textbf{48.4}$ & $\textbf{55.0}$ & $\underline{55.1}$ & $60.4$ & $66.7$ & $\textbf{70.3}$ & $\textbf{69.3}$ & $\textbf{70.9}$ & $\textbf{71.1}$ & $\textbf{71.9}$ & $\underline{75.0}$ & $77.2$ \\
    \hline
\end{tabular}}
\end{table*}
%%%%%%%%%%%%%%%%%%%%%%%%%%%%%%%%%%%%%%%%%%%%%%%%%%%%%%%%%%%%%%

%%%%%%%%%%%%%%%%%%%%%%%%%%%%%%%%%%%%%%%%%%%%%%%%%%%%%%%%%%%%%%
\begin{table*}[ht]
\renewcommand{\arraystretch}{1.3}
\caption{Classification accuracy (\%) for three large-scale graph
datasets, where the column of labeled examples is measured in terms of
percentages of the entire dataset and OOM means ``out of
memory".}\label{tab:results_large_scale}
\centering
\begin{tabular}{c|cccc|cccc|cccc} \hline
 & \multicolumn{4}{c}{\textbf{Reddit}} & \multicolumn{4}{|c}{\textbf{Amazon2M}} & \multicolumn{4}{|c}{\textbf{PPI}}\\ \hline
 \% of labeled examples & $1$ & $2$ & $5$ & $10$ & $1$ & $2$ & $3$ & $5$ & $1$ & $2$ & $5$ & $10$\\ \hline \hline
 \textbf{FastGCN} & $78.6$ & $80.2$ & $86.7$ & $87.7$ & OOM & OOM & OOM & OOM & - & - & - & -  \\ \hline
 \textbf{Cluster-GCN} & \underline{92.0} & \underline{92.7} & \underline{93.7} & \underline{94.0} & - & - & 
 $\textbf{68.8}$ & $\textbf{73.8}$ & - & - & - & -\\ \hline
 \textbf{L-GCN} & $89.2$ & $90.7$ & $92.0$ & $92.8$ & \underline{36.9} & \underline{47.6} 
 & $59.4$ & $65.8$ & \underline{68.3} & \underline{69.8} & \underline{69.2} & \underline{69.1} \\ \hline \hline
 \textbf{GraphHop} & $\textbf{93.4}$ & $\textbf{94.1}$ & $\textbf{94.7}$ & $\textbf{95.0}$ & $\textbf{51.4}$ & $\textbf{59.8}$ 
 & \underline{66.3} & \underline{68.2} & $\textbf{73.8}$ & $\textbf{73.4}$ & $\textbf{73.7}$ & $\textbf{74.1}$  \\ \hline
\end{tabular}
\end{table*}
%%%%%%%%%%%%%%%%%%%%%%%%%%%%%%%%%%%%%%%%%%%%%%%%%%%%%%%%%%%%%%

We conduct performance benchmarking between GraphHop and several
state-of-the-art methods and compare their results for small- and
large-scale datasets below. 

\subsubsection{Citation Networks}\label{subsubsec:citation}

The state-of-the-art methods used for performance benchmarking
are grouped into three categories: 
\begin{itemize}
\item Unsupervised embedding methods: LP \cite{10.5555/3041838.3041953},
DeepWalk \cite{perozzi2014deepwalk}, LINE \cite{Tang_2015}, DGI
\cite{48921} and Graph2Gauss \cite{bojchevski2017deep}). 
\item Graph neural network (GNN) based methods: GCN
\cite{kipf2016semi} and GAT \cite{velickovic2018graph}. 
\item GCN with self-training: Co-training GCN and Self-training GCN
\cite{li2018deeper,sun2019multi}. 
\end{itemize}
Results on three citation datasets are summarized in Table
\ref{tab:results_citation}, where label rates are chosen to be 1, 2, 4,
8, 16 and 20 labeled nodes per class.  Each column shows the
classification accuracy (\%) for GraphHop and 9 benchmarking methods
under a dataset and a given label rate.

Overall, GraphHop performs the best. This is especially true for cases
with extremely small label rates. This is because its adoption of label
aggregation and LR classifier update yields a smooth distribution of
label embeddings on the graph for prediction, which relies less on label
supervision. Similarly, other embedding-based methods (e.g.  DGI and
Graph2Gauss) also outperform GNN variants for cases with very few labels
since their methods are designed to take advantage of the graph
structure into embeddings in an unsupervised way. When the label rate
goes higher, GNN variants perform better than unsupervised models. 
Li {\em et al.} \cite{li2018deeper} and Sun {\em et al.}
\cite{sun2019multi} showed the limitation of GCN in a few label cases and
proposed a co-training or self-training mechanism to handle this
problem.  Still, GraphHop outperforms their methods in various label
rates.  

\subsubsection{Large-Scale Graphs}\label{subsubsec:large}

To demonstrate the scalability of GraphHop, we apply it to three large-scale
graph datasets.  To save time and memory, only one-hop label embedding
is propagated at each iteration. We compare GraphHop with three
state-of-the-art GCN variants designed for large-scale graphs They are:
FastGCN \cite{chen2018fastgcn}, Cluster-GCN \cite{chiang2019cluster},
and L-GCN \cite{you2020l2}. For the three benchmarking methods, we tried
our best in adopting their released codes and following the same
settings as described in their papers. However, their original models
deal with supervised learning in an inductive setting. For a fair
comparison, we modify their model input so that the training is
conducted on the entire graph. 

The classification accuracy results for node label rates equal to 1\%,
2\%, 5\% and 10\% of the entire graph are summarized in Table
\ref{tab:results_large_scale}.  For the Reddit dataset, we can get
results for all methods. We see from the table that GraphHop performs
the best and Cluster-GCN the second in all cases.  For the super large
Amazon2M dataset, FastGCN exceeds the memory while Cluster-GCN fails to
converge when the label rates are 1\% and 2\%.  Furthermore, it takes
Cluster-GCN long time for one epoch training and we report the results
after $20$ epochs. There is a significant performance gap between
GraphHop and L-GCN in all cases for Amazon2M.  For the PPI multi-label
dataset, both FastGCN and Cluster-GCN fail to converge in the
transductive semi-supervised setting.  As the label rates increase in
PPI, the performance stays about the same. This is probably because a
small label rate (less than 10\%) is not large enough to cover all cases
in a multi-label scenario. 

%%%%%%%%%%%%%%%%%%%%%%%%%%%%%%%%%%%%%%%%%%%%%%%%%%%%%%%%%%%%%%
\begin{figure*}[ht]
\begin{center}
\includegraphics[width=0.8\linewidth]{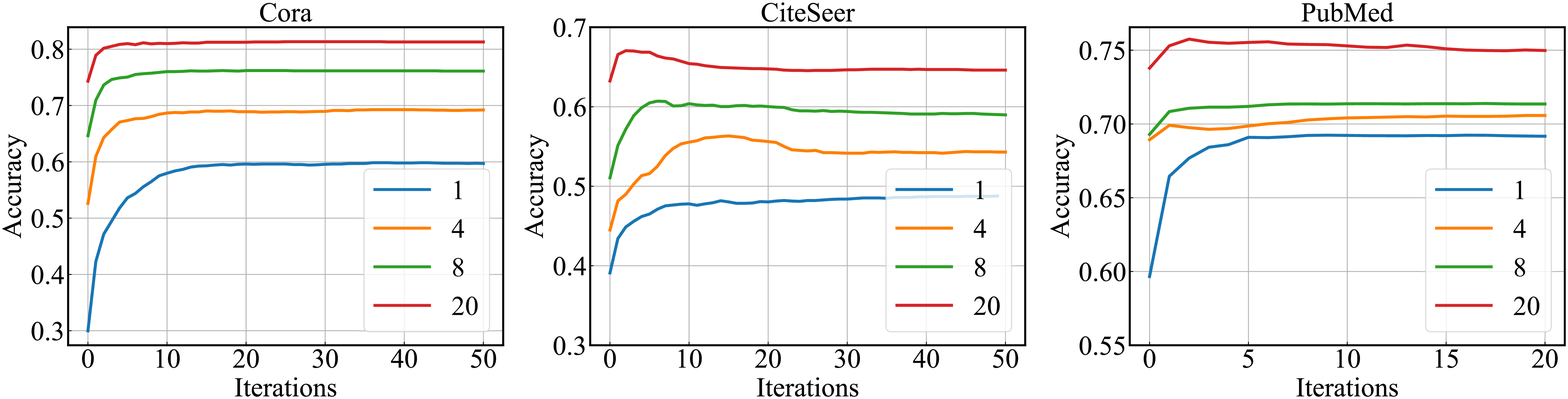}\\
\includegraphics[width=0.8\linewidth]{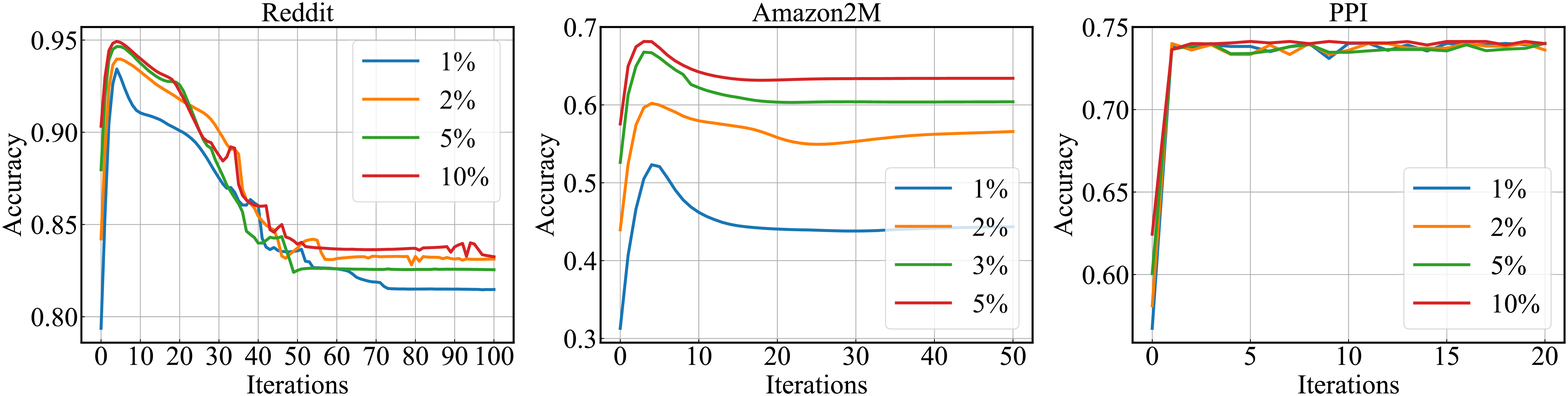}
\end{center}
\caption{GraphHop's test accuracy curves as a function of the iteration
number parameterized by different label rates, where the label rate is
expressed as the number and the percentage of labeled nodes per class for 
citation and large-scale graph datasets, respectively.} \label{fig:test_accuracy}
\end{figure*}
%%%%%%%%%%%%%%%%%%%%%%%%%%%%%%%%%%%%%%%%%%%%%%%%%%%%%%%%%%%%%%

%%%%%%%%%%%%%%%%%%%%%%%%%%%%%%%%%%%%%%%%%%%%%%%%%%%%%%%%%%%%%%
\begin{figure*}[ht]
\begin{center}
\includegraphics[width=0.8\linewidth]{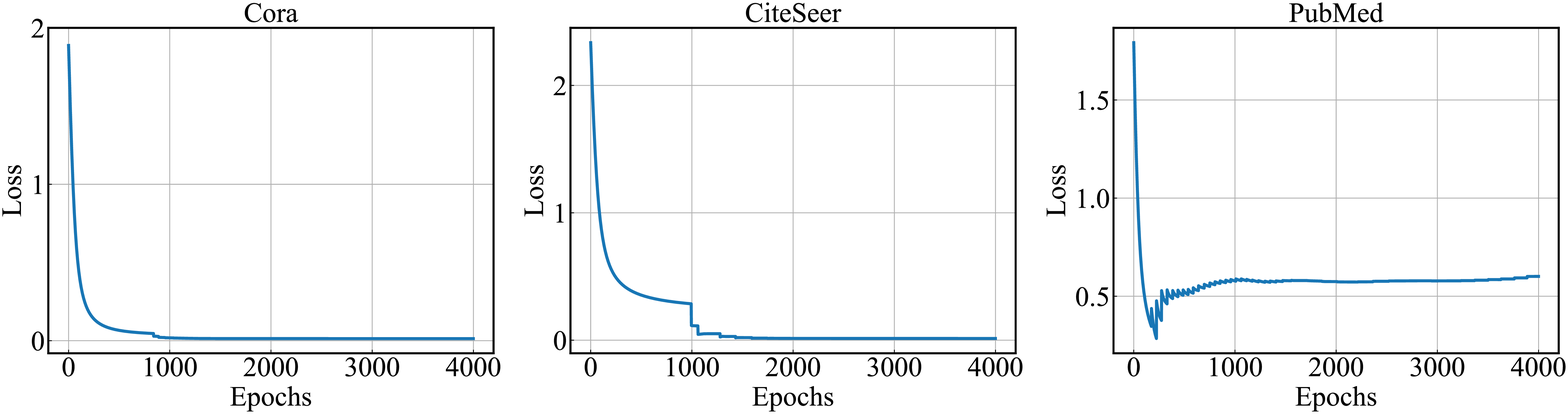}\\
\includegraphics[width=0.8\linewidth]{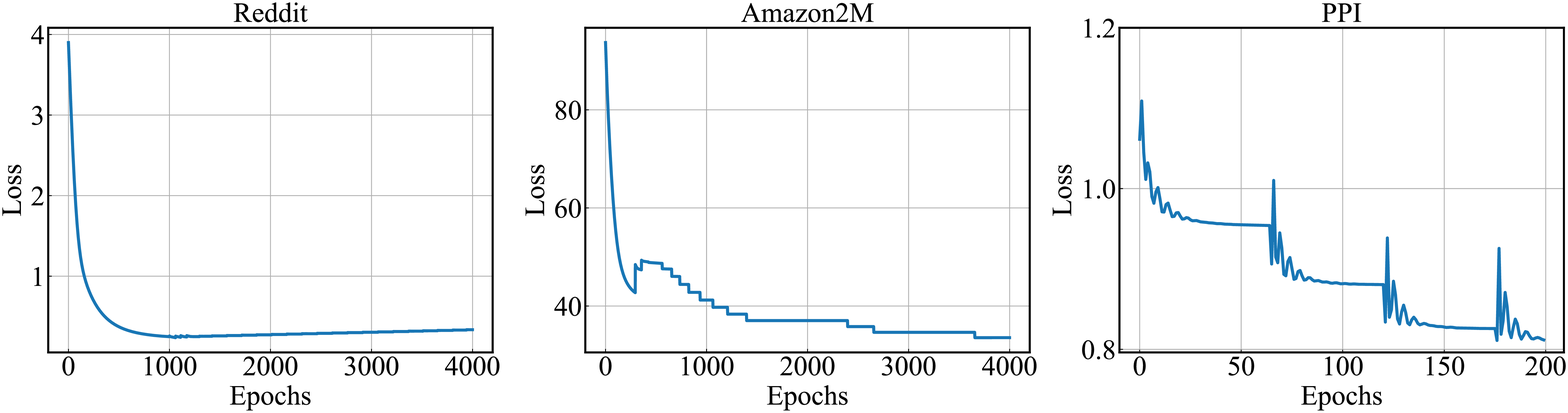}
\end{center}
\caption{Training loss curves of the LR classifiers of GraphHop
as a function of the epoch number for six datasets.}\label{fig:train_loss}
\end{figure*}
%%%%%%%%%%%%%%%%%%%%%%%%%%%%%%%%%%%%%%%%%%%%%%%%%%%%%%%%%%%%%%

%%%%%%%%%%%%%%%%%%%%%%%%%%%%%%%%%%%%%%%%%%%%%%%%%%%%%%%%%%%%%%
\begin{figure}[ht]
\centering
\includegraphics[width=0.7\linewidth]{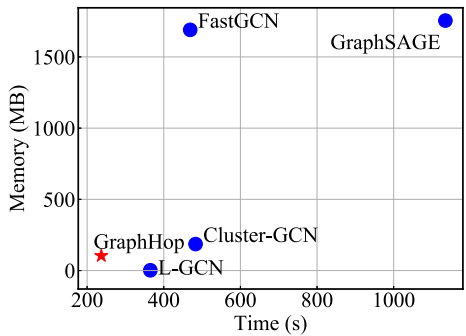}
\caption{Comparison of time complexity vs. memory usage of different 
methods on Reddit, where the lower left corner indicates the desired 
region that has low training complexity and low GPU memory consumption.} 
\label{fig:results_time_memory_a}
\end{figure}
%%%%%%%%%%%%%%%%%%%%%%%%%%%%%%%%%%%%%%%%%%%%%%%%%%%%%%%%%%%%%%

%%%%%%%%%%%%%%%%%%%%%%%%%%%%%%%%%%%%%%%%%%%%%%%%%%%%%%%%%%%%%%
\begin{figure}[ht]
\centering
\includegraphics[width=0.8\linewidth]{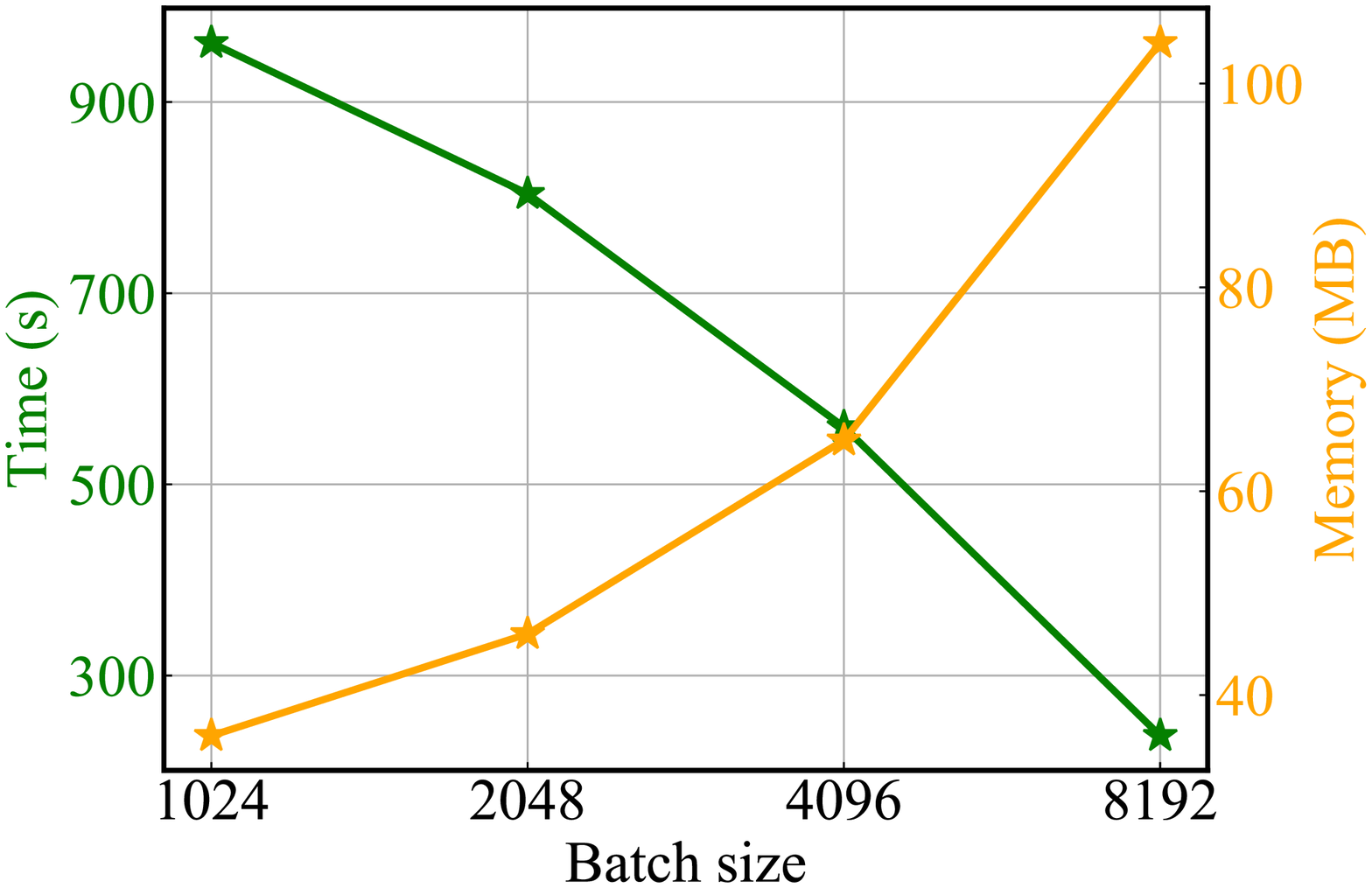} \\
\caption{The trade-off between training time and memory usage by varying
the training batch size on Reddit for GraphHop, where the green curve
indicates the training time and the orange curve indicates the memory
usage.} \label{fig:results_time_memory_b}
\end{figure}
%%%%%%%%%%%%%%%%%%%%%%%%%%%%%%%%%%%%%%%%%%%%%%%%%%%%%%%%%%%%%%

\subsection{Computational Complexity and Memory Requirement}

Since GraphHop is an iterative algorithm, we study test accuracy curves
as a function of the iteration number for all six datasets with
different label rates in Fig.  \ref{fig:test_accuracy}.  We have two
main observations. First, for Cora, CiteSeer, and PubMed datasets,
test accuracy curves converge in about 10, 5, and 4 iterations,
respectively.  Second, for Reddit and Amazon2M, test accuracy
curves achieve their peaks in a few iterations, drop and then converge.
The latter is due to over-smoothing of label signals as explained in
Sec.  \ref{subsec:ablation_study}. Repeatedly applying smoothing operators will result in the convergence of feature vectors to the same values \cite{li2018deeper}. It can be alleviated by residual
connections \cite{he2016deep}. In practice, we can find the optimal
iteration number for each dataset by observing the convergence behavior
based on the evaluation data. 

Furthermore, we need to train an LR classifier for label embedding
initialization and propagation through iterative optimization.  To
achieve this objective, we show the training loss curves of the LR
classifiers as a function of the epoch number for all six datasets in
Fig.  \ref{fig:train_loss}.  We see that all LR classifiers converge
relatively fast.  Although fewer labeled nodes tend to demand more
iterations to achieve convergence, yet its impact on convergence
behavior is minor, which is consistent with Corollary \ref{corollary}. 

To demonstrate the efficiency and scalability of GraphHop, we compare
training time and memory usage of several methods in Table
\ref{tab:results_time_memory}. Here, we focus on benchmarking models
that can handle large-scale graphs such as GraphSAGE
\cite{hamilton2017inductive}, Cluster-GCN \cite{chiang2019cluster},
L-GCN \cite{you2020l2} and FastGCN \cite{chen2018fastgcn}.  For citation
networks of smaller sizes, we adopt their original codes and implement
them in a supervised way. For large-scale graphs (i.e., Reddit, PPI, and
Amazon2M), we follow the process discussed earlier and set the label
rate to the largest. We measure the averaged running time per epoch (or per
iteration for GraphHop) and the total training time in seconds. An early
stopping is adopted.  That is, we record the time when the performance
on the validation set drops continuously for five iterations.  For memory
usage, we only consider the GPU memory \footnote{It is measured by
\textsf{{torch.cuda.memory\_allocated()}} for PyTorch and
\textsf{{tf.contrib.memory\_stats.{MaxBytesInUse}()}} for
TensorFlow. }.

%%%%%%%%%%%%%%%%%%%%%%%%%%%%%%%%%%%%%%%%%%%%%%%%%%%%%%%%%%%%%%
\begin{table*}[ht]
\renewcommand{\arraystretch}{1.3}
\caption{Comparison of training time and GPU memory usage. 
The averaged running time per epoch/iteration and the total running 
time are given outside and inside the parantheses, respectively.}\label{tab:results_time_memory}
\centering
\resizebox{0.98\textwidth}{!}{
\begin{tabular}{c|cc|cc|cc|cc|cc}
    \hline
    \multirow{2}{*}{} & \multicolumn{2}{c|}{\textbf{GraphSAGE}} & \multicolumn{2}{c|}{\textbf{Cluster-GCN}} & \multicolumn{2}{c|}{\textbf{L-GCN}} & \multicolumn{2}{c|}{\textbf{FastGCN}} & \multicolumn{2}{c}{\textbf{GraphHop}} \\
    \cline{2-11}
    & Time & Memory & Time & Memory & Time & Memory & Time & Memory & Time & Memory  \\
    \hline
    \textbf{Cora} & $0.033 (7)$s & $902$ MB & $0.142 (39)$s & $546$ MB & $0.004 (0.7)$s & $\textbf{3}$ MB & $0.023 (3)$s & $\underline{21}$ MB & $0.010 (36)$s & $26$ MB \\
    \hline
    \textbf{CiteSeer} & $0.059 (12)$s & $2288$ MB & $0.175 (56)$s & $723$ MB & $0.009 (1.4)$s & $\textbf{5}$ MB & $0.055 (4)$s & $74$ MB & $0.018 (56)$s & $\underline{68}$ MB \\
    \hline
    \textbf{PubMed} & $0.022 (5)$s & $418$ MB & $0.483 (148)$s & $808$ MB & $0.012 (1.9)$s & $\textbf{4}$ MB & $0.214 (7)$s & $81$ MB & $0.058 (78)$s & $\underline{9}$ MB  \\
    \hline
    \textbf{Reddit} & $5.7 (1135)$s & $1755$ MB & $3.3 (483)$s & $186$ MB & $2.3 (365)$s & $\textbf{2}$ MB & $13.7 (469)$s & $1690$ MB & $11.3 (237)$s & $\underline{104}$ MB\\
    \hline
    \textbf{Amazon2M} & $44.6 (913)$s & $2167$ MB & $1010 (10251)$s & $\underline{73}$ MB & $3.4 (549)$s & $\textbf{3}$ MB & OOM & OOM & $69.3 (762)$ & $302$ MB \\
    \hline
    \textbf{PPI} & $1.3 (26)$s & $110$ MB & - & - & $0.099 (16)$s & $\textbf{14}$ MB & - & - & $0.077 (46)$s & $\underline{21}$ MB \\
    \hline
\end{tabular}}
\end{table*}
%%%%%%%%%%%%%%%%%%%%%%%%%%%%%%%%%%%%%%%%%%%%%%%%%%%%%%%%%%%%%%

Generally speaking, GraphHop can achieve fast training with low memory
usage.  Although L-GCN has the lowest memory usage, all parameters are
fixed without validation applied (validation data are counted in memory consumption for ours and other baselines) so that the comparison may not be
fair. To shed light on training complexity and memory usage, we plot the
time complexity versus memory usage of different methods on Reddit in
Fig. \ref{fig:results_time_memory_a} based on the data in Table
\ref{tab:results_time_memory}.  The lower-left corner of this figure
indicates the desired region that has low training complexity and low
GPU memory consumption. Furthermore, we can balance the training time
and memory usage by changing the batch size as shown in Fig.
\ref{fig:results_time_memory_b} for GraphHop. By increasing the batch
size, the memory consumption goes up in exchange for lower training
time. 

\subsection{Additional Observations}\label{subsec:ablation_study}

\subsubsection{Ablation Study} 

We explore two variants of GraphHop below.
\begin{itemize}
\item Variant I: GraphHop with the initialization stage only. \\
It utilizes center's and neighbors' node features for label
prediction without any label embeddings propagation 
\item Variant II: GraphHop without the initialization stage and LR
classification between iterations, which is the same as vanilla LP. \\
It propagates label embeddings without leveraging any feature
information in the initialization stage. 
\end{itemize}
We compare the test accuracy values of all three methods with the same label rate in Table
\ref{tab:results_variants}.  Clearly, GraphHop outperforms the other two
by significant margins.  Furthermore, we show test accuracy curves as a
function of the iteration number for Variant II and GraphHop in Fig.
\ref{fig:results_variants}. We see that GraphHop not only achieves
higher accuracy but also converges faster than Variant II. It indicates
the importance of good initialization of label embeddings using node
features and their update by classifiers in each iteration. 
Label aggregation and update can improve the performance of GraphHop 
furthermore.

%%%%%%%%%%%%%%%%%%%%%%%%%%%%%%%%%%%%%%%%%%%%%%%%%%%%%%%%%%%%%%
\begin{table}[ht]
\renewcommand{\arraystretch}{1.3}
\caption{Comparison of accuracy performance of GraphHop and its
two variants.} \label{tab:results_variants}
\centering
\begin{tabular}{ccccc} \hline
     & \textbf{Cora} & \textbf{CiteSeer} & \textbf{PubMed} \\ \hline
    \textbf{Variant I} & $74.9$ & $64.0$ & $75.3$   \\
    \textbf{Variant II} & $67.3$ & $44.8$ & $66.4$  \\
    \textbf{GraphHop} & $\textbf{81.0}$ & $\textbf{70.3}$ & $\textbf{77.2}$ \\ \hline
\end{tabular}
\end{table}
%%%%%%%%%%%%%%%%%%%%%%%%%%%%%%%%%%%%%%%%%%%%%%%%%%%%%%%%%%%%%%

%%%%%%%%%%%%%%%%%%%%%%%%%%%%%%%%%%%%%%%%%%%%%%%%%%%%%%%%%%%%%%
\begin{figure*}[ht]
\centering
\includegraphics[width=0.8\linewidth]{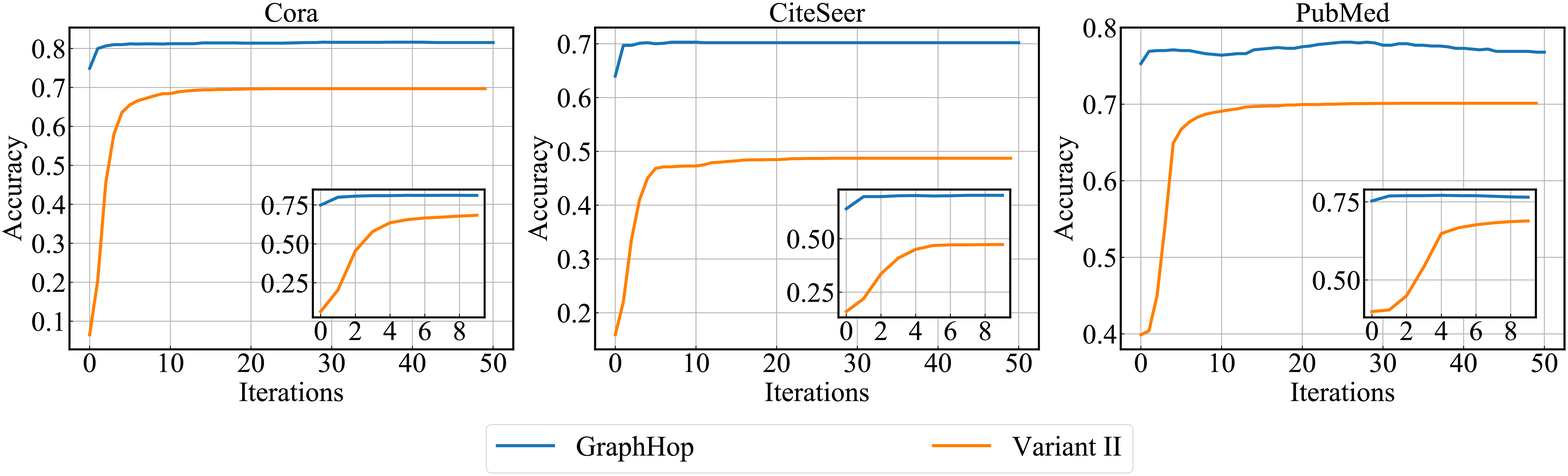}
\caption{The convergence curves of GraphHop and Variant II on
citation datasets, where inner figures show the initial 10 
iterations.}\label{fig:results_variants}
\end{figure*}
%%%%%%%%%%%%%%%%%%%%%%%%%%%%%%%%%%%%%%%%%%%%%%%%%%%%%%%%%%%%%%

%%%%%%%%%%%%%%%%%%%%%%%%%%%%%%%%%%%%%%%%%%%%%%%%%%%%%%%%%%%%%%
\begin{figure}[ht]
\centering
\includegraphics[width=0.45\textwidth]{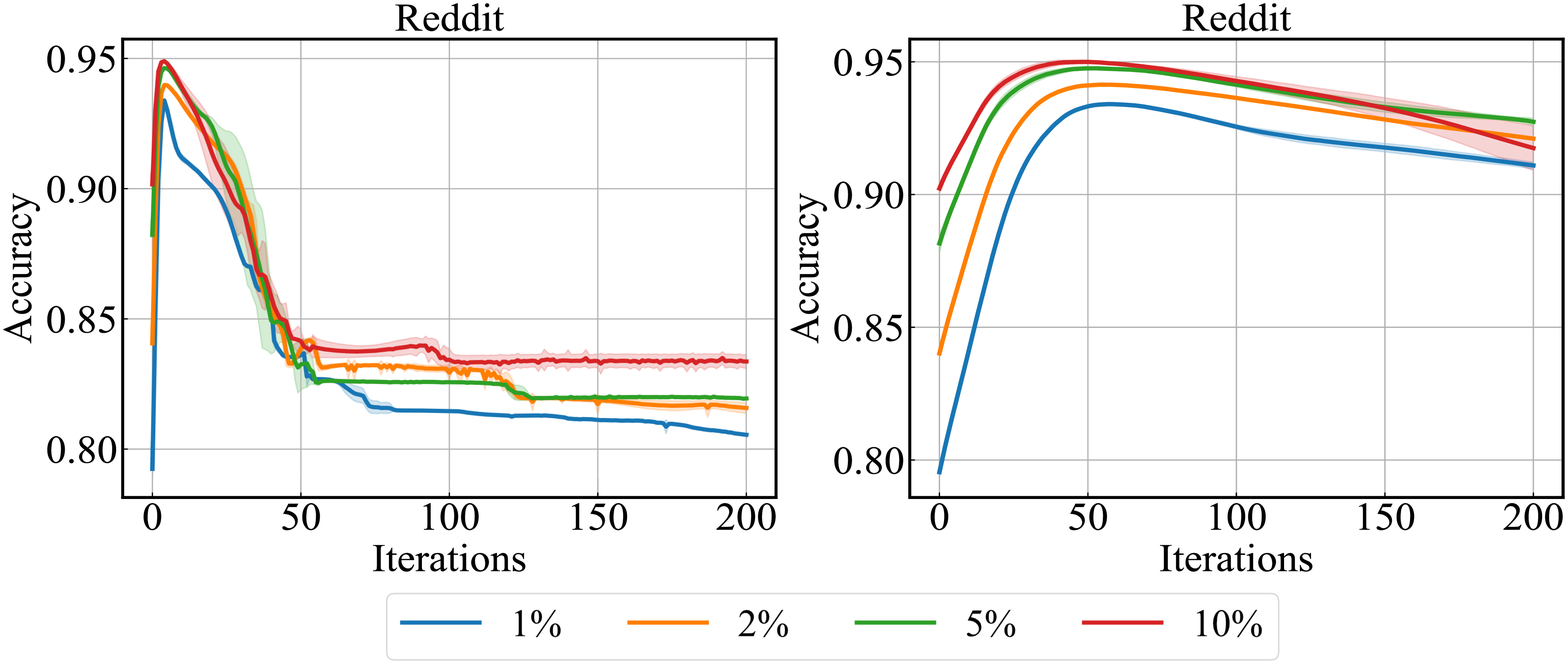}
\caption{The convergence of GraphHop's test accuracy curves with (right) 
and without (left) residual connections for Reddit, where shaded areas 
indicate the standard deviation range.} \label{fig:results_residual}
\end{figure}
%%%%%%%%%%%%%%%%%%%%%%%%%%%%%%%%%%%%%%%%%%%%%%%%%%%%%%%%%%%%%%

%%%%%%%%%%%%%%%%%%%%%%%%%%%%%%%%%%%%%%%%%%%%%%%%%%%%%%%%%%%%%%
\begin{figure}[!ht]
\centering
\includegraphics[width=0.45\textwidth]{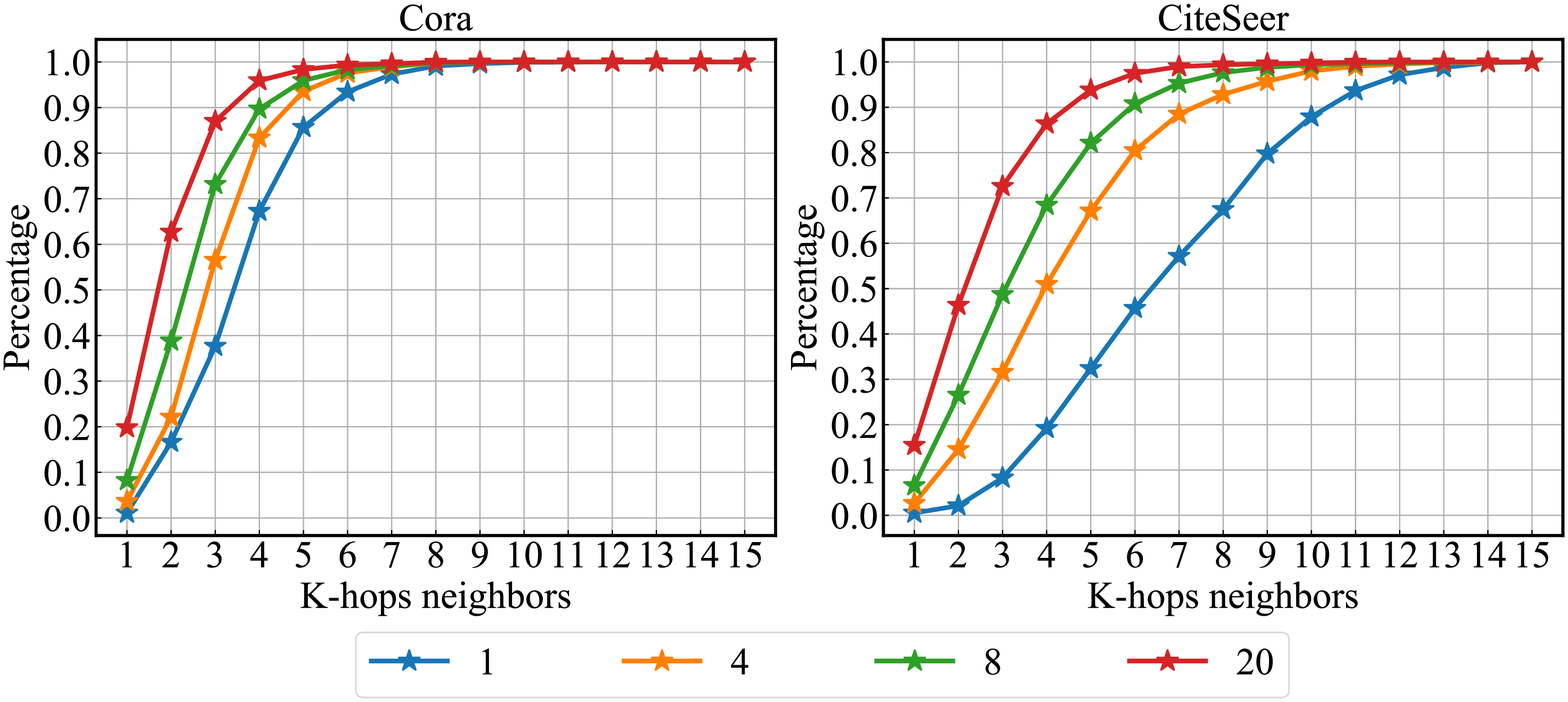}
\caption{Plots of the cumulative percentages of nodes in subset
$\mathcal{V}_1\cup\mathcal{V}_2\cup...\cup\mathcal{V}_k$, where different
colors indicates different label rates.} \label{fig:hops_number}
\end{figure}
%%%%%%%%%%%%%%%%%%%%%%%%%%%%%%%%%%%%%%%%%%%%%%%%%%%%%%%%%%%%%%

\subsubsection{Over-Smoothing Problem}

We study the over-smoothing problem by examining the convergence
behavior of GraphHop on Reddit in Fig. \ref{fig:results_residual}.  The
left subfigure shows the averaged convergence curves of GraphHop on
Reddit with multiple label rates. The curves drop after 5 iterations and
converge at around 50 iterations. We argue that this phenomenon is due
to over-smoothing of label embeddings. Generally speaking, correlations
of embeddings are valid only locally. This is especially true for
large-scale graphs.  Adding uncorrelated information from long-distance
hops tends to have negative impact.  To verify this claim, we conduct
experiments on a variant of GraphHop, whose label embeddings are updated
in form of 
\begin{equation}\label{eq:residual}
\vect{H}^{(t + 1)} = (1 - \tau) \mathrm{p}_{\text {model}}(\vect{Y}|\vect{H}^{(t)}; \vect{\theta}) 
+ \tau \vect{H}^{(t)},
\end{equation}
where parameter $\tau \in (0, 1)$ is used to control the update speed.
Eq. (\ref{eq:residual}) is also known as the residual connection
\cite{he2016deep}. A large $\tau$ value enables the model to preserve
more information from the previous iteration and slows down the
smoothening speed. 
We report the results with $\tau = 0.9$ in the right subfigure while
keeping the other settings the same as the left subfigure. We observe
more stable curves with slower performance degradation. 

\subsubsection{Fast Convergence} 

We explain why convergence can be achieved by only a few iterations.
This is because a small number of iterations reaches sufficiency in
Corollary \ref{corollary}. Without loss of generality, we choose Cora
and CiteSeer datasets for illustration.  We calculate the number of
nodes in each subset, $\mathcal{V}_i$, in Eq.  (\ref{equ:theorem}) up to
$k$ hops.  The cumulative results are shown in Fig.
\ref{fig:hops_number}.  Both subfigures indicate that all nodes in the
graph can be reached from labeled examples within less than 10 hops.
Fast propagation of embeddings yields efficient label smoothing and
fast convergence. Also, we see from Fig.  \ref{fig:hops_number} that
higher label rates will reach sufficient iterations faster than lower
label rates. 

\section{Comments on Related Work}\label{sec:review}

\textbf{Semi-supervised Learning.} There is rich literature on
semi-supervised learning \cite{chapelle2009semi}, including generative
models \cite{kingma2014semi}, the transductive support vector machine
\cite{bennett1999semi}, entropy regularization
\cite{grandvalet2005semi}, manifold learning \cite{belkin2006manifold}
and graph-based methods \cite{zhou2003learning, liang2018lightweight,
iscen2019label, gong2015deformed}. Our discussion is restricted to
graph-related work.  Most semi-supervised graph-based methods are built
on the manifold assumption \cite{chapelle2009semi}, where nearby nodes
are close in the data manifold and, as a result, they tend to have the
same labels.  Early research penalizes non-smoothness along edges of a
graph with the Markov random field \cite{zhu2002towards}, the Laplacian
eigenmaps \cite{belkin2004semi}, spectral kernels
\cite{chapelle2002cluster}, and context-based methods
\cite{grover2016node2vec, perozzi2014deepwalk}.  Their main difference
lies in the choice of regularization. The most popular one is the
quadratic penalty term applied on nearby nodes to enforce label
consistency with the data geometry. The optimization result is shown to
be equivalent to LP \cite{chapelle2006label}.  Traditional graph methods
are non-parametric, discriminative, and transductive in nature, which
make them lightweight with good classification performance. To further
improve the performance, methods are developed by combining graph-based
regularization with other models
\cite{sen2008collective,iscen2019label}. Rather than regularizing on
label embeddings, this is extended to attributes \cite{weston2012deep}
and even to hidden layers or auxiliary embeddings in neural networks.
Manifold regularization \cite{belkin2006manifold} and Planetoid
\cite{yang2016revisiting} generalize the Laplacian regularizer with a
supervised classifier that imposes stronger constraints on the model
learning. 

\textbf{Graph Neural Networks.} Inspired by the recent success of
convolutional neural networks (CNNs) \cite{lecun2015deep,
lecun1995convolutional} on images and videos, a series of efforts have
been made to generalize convolutional filters from grid-structured
domains to non-Euclidean domains \cite{henaff2015deep,
bruna2013spectral, atwood2016diffusion, duvenaud2015convolutional} with
theoretical support from graph signal processing
\cite{shuman2013emerging}. The space spanned by the eigenvectors of the
graph Laplacian can be regarded as a generalization of the Fourier
basis. By following this idea, a deep neural architecture was formulated
in \cite{henaff2015deep, bruna2013spectral} to employ the Fourier
transform as a projection onto the eigenbasis of the graph Laplacian.
To overcome the expensive eigendecomposition, recurrent Chebyshev
polynomials were proposed in \cite{defferrard2016convolutional} as an
efficient filter for approximation. GCNs \cite{kipf2016semi} further
simplified it by only considering the first-order approximation in the
Chebyshev polynomials.  GCN has inspired quite a lot of follow-up work, e.g.,
\cite{48921, xu2018powerful, spinelli2020adaptive, wu2020comprehensive}. 

With the combination of layer-wise design and nonlinear activation, GCNs
offer impressive results on the semi-supervised classification problem.
Later, it was explained in \cite{wu2019simplifying, li2019label,
nt2019revisiting, li2018deeper} that the success of GCNs is due to a
lowpass filtering operation performed on node attributes. 
Specifically, \cite{li2019label, wu2019simplifying, nt2019revisiting}
have shown that the powerful feature extraction ability behind the graph
convolutional operation in GCNs is due to a low-pass filter applied on
feature matrix to extract only smooth signals for prediction. This
simplified graph convolution can be formulated as an LR
classifier on the aggregated features \cite{wu2019simplifying}.
That is, we have
\begin{equation}\label{equ:simplify_gcn}
\vect{H} = \mathrm{p}_\text{model}(\vect{Y}|\vect{S}^k\vect{X}; \vect{\theta}),
\end{equation}
where $\mathrm{p}_\text{model}(\cdot)$ is the LR classifier, $\vect{S}^k$ is
the $k$-th power of the normalized adjacency matrix, $\vect{S}$, $\vect{\theta}$ is
the classifier parameters, and $\vect{H}$ is the label embeddings output as
defined previously. This shed light on a simple filter design on graph
feature matrix. However, in order to make a stronger low-pass filter,
the power $k$ should be large, which induces an expensive cost $O(kn^3)$
in multiplying adjacency matrix and further be scaled to large graphs.
Besides, there is no restriction on the smoothness of the predicted
label embeddings. 

Despite the strength of GCNs, they are limited in two aspects: 1)
effective modeling of node attributes and labels jointly; 2) ease of
scalability to large graphs. For the first point, unlabeled examples are
not integrated into model training but only inference.  Several algorithms
have been proposed to tackle this deficiency. Zhang {\em et al.}
\cite{zhang2019bayesian} employed a Bayesian approach by modeling the
graph structure, node attributes, and labels as a joint probability and
inferring the unlabeled examples by calculating the posterior
distribution. On the other hand, Qu {\em et al.} \cite{qu2020gmnn} used
the conditional random field to embed the correlation between labeled
and unlabeled examples with GCNs for feature extraction.  In practice,
these methods are costly and incurred by the local minimum during
optimization. Another line of methods employed self-training techniques
to generate pseudo-labels for unlabeled examples and used them
throughout training \cite{li2018deeper, you2020does, sun2019multi}.
However, they do not utilize the correlation between labeled and
unlabeled examples effectively and often suffer from label error
feedback \cite{lee2013pseudo}. Recently, active learning methods
\cite{li2020seal, yu2018active} have been proposed, which directly query
for labeled examples and train the model in an incremental way. They
are more sophisticated in model design and more difficult to optimize. 

For the second point, one main issue of GCNs is the large number of
model parameters, which make their optimization very difficult for large
graphs. Unlike images in computer vision or sentences in natural
language processing, one graph can be large in its size while its nodes
are connected without segmentation.  The layer-wise convolutional
operation introduces an exponential expansion of neighborhood sizes
\cite{chiang2019cluster}, which hinders GCNs from batch training.
Sampling-based strategies (e.g. GraphSAGE \cite{hamilton2017inductive}
and FastGCN \cite{chen2018fastgcn}) have been proposed to overcome this
problem. They attempt to reduce the neighborhood size during
aggregation. Alternatively, some methods \cite{chiang2019cluster,
zeng2019graphsaint} directly sample one or more subgraphs and perform
subgraph-level training.  Recently, You {\em et al.} \cite{you2020l2}
proposed a layer-wise training algorithm for GCNs, which is called
L-GCN.  The idea is that, instead of training multiple GCN layers at
once, the gradient update and parameters convergence are performed in a
layer-wise fashion. Yet, L-GCN requires a large amount of training data
(i.e., heavily supervised learning) and its performance degrades
dramatically when only a few labeled examples are available in training. 

\section{Conclusion}\label{sec:conclusion}

A novel iterative method, called GraphHop, was proposed for transductive
semi-supervised node classification on graph-structured data. It first
initializes label embedding vectors through predictions based on node
features. Then, based on the smooth label assumption on graphs, GraphHop
propagated label embeddings iteratively to regularize the initial
inference furthermore. Instead of directly replacing the embeddings by
neighbor averages at each iteration, GraphHop adopts an LR classifier
between iterations to align embeddings between neighbors and itself
effectively. Incorporation of label initialization, label aggregation
from neighbors, and label update via classifiers makes GraphHop both
efficient and lightweight. Experimental results on various scales of
networks demonstrate that the proposed GraphHop method offers an
effective semi-supervised solution to node classification, which is
especially true in extremely few labeled examples. 

\appendices

\section{Implementation of LR Classifier}\label{sec:LR}

We explain the implementation details of the LR classifier here.
For labeled examples, the supervised loss term can be written as
\begin{equation}\label{equ:label_loss}
L_l =\frac{1}{|\mathcal{L}|} \sum_{\substack{y \in \mathcal{L} \\ 
\vect{h}_M^{(t - 1)} \in \vect{H}^{(t - 1)}_{M, l}}} \mathrm{H}\left(y, 
\mathrm{p}_{\text {model}}(y \mid \vect{h}_M^{(t - 1)} ; \vect{\theta})\right),
\end{equation}
where $\mathrm{H}(p, q)$ is the entropy loss for classification and
$\vect{\theta}$ denotes the parameters of the classifier. The missing labels
for unlabeled examples prevent the direct supervision. However, we can
leverage the label embedding from node itself as supervision and treat
propagated embeddings as input. The explicit results are the probability
predictions from the classifiers that are consistent between neighbors and
node itself. 

Similar ideas can be employed as regularization, so-called consistency
regularization \cite{tarvainen2017mean, berthelot2019mixmatch, 8417973},
which enforces model predictions to be consistent under any input
transformations. Then, the loss term for unlabeled examples can be
expressed as
\begin{equation}\label{equ:unlabel_loss}
L_u =\frac{1}{|\mathcal{U}||C|} \sum_{\substack{\vect{h}^{(t - 1)} 
\in \vect{H}^{(t - 1)}_u \\ 
\vect{h}_M^{(t - 1)} \in \vect{H}^{(t - 1)}_{M, u}}} \mathrm{H}\left(\vect{h}^{(t -
1)}, \mathrm{p}_{\text {model }}(y \mid \vect{h}_M^{(t - 1)} ;
\vect{\theta})\right).
\end{equation}

Given the loss terms for labeled and unlabeled examples in Eqs.
(\ref{equ:label_loss}) and (\ref{equ:unlabel_loss}), respectively, 
the final loss function is a weighted sum of the two,
\begin{equation}\label{equ:label_unlabel_loss}
L = L_l + \alpha L_u,
\end{equation}
where $\alpha$ is a hyperparameter. The LR classifier are trained in
multiple epochs until convergence. Finally, new label embeddings are
predicted by 
\begin{equation}\label{equ:infer}
\vect{H}^{(t)} = \mathrm{p}_\text{model}(\vect{Y}|\vect{H}_M^{(t - 1)};\vect{\theta}).
\end{equation}

Before training the LR classifier, we perform one additional step
inspired by entropy minimization \cite{grandvalet2005semi} and knowledge
distillation \cite{hinton2015distilling} for
supervision in Eq. (\ref{equ:unlabel_loss}). That is, given label
embeddings, we apply a sharpening function to adjust the entropy of the
label distribution. A temperature is introduced to alternate the
categorical distribution, which is defined as
\begin{equation}
\text{Sharpen}(p, T)_i = p_i^{\frac{1}{T}} \bigg/ \sum_{j = 1}^C p_j^{\frac{1}{T}},
\end{equation}
where $p$ is the input categorical distribution (i.e. label embeddings
$\vect{H}$ in GraphHop) and $T$ is the temperature hyperparameter. Using a
higher value for $T$ produces a softer probability distribution over
classes, and vice versa. By adjusting the temperature, models can decide
how confidence they should believe in the current label embeddings
iteration. 

The sharpening operation may result in uniform label distributions by
setting a large temperature. To enforce the classifier outputs
low-entropy predictions on unlabeled data, we minimize the
entropy \cite{berthelot2019mixmatch, grandvalet2005semi} of model
prediction $\mathrm{p}_\text{model}(y|x;\theta)$ with an additional 
loss term
\begin{equation}
L_{u}^{\prime}=\frac{1}{|\mathcal{U}||C|} \sum_{\vect{h}^{(t -
1)}_M \in \vect{H}_{M, u}^{(t - 1)}}
\mathrm{H}\left(\mathrm{p}_{\text{model}},
\mathrm{p}_\text{model}\right).
\end{equation}
In practice, the final loss can be written as
\begin{equation}
L = L_{l} + \alpha L_{u}+\beta 
L_{u}^{\prime},
\label{equ:final_loss}
\end{equation}
where $\beta$ is a hyperparameter to adjust the scale of the entropy.

%\section*{Acknowledgment}
%\ifCLASSOPTIONcaptionsoff
% \newpage
%\fi

\bibliographystyle{IEEEtran}
\bibliography{graphhop}

\begin{thebibliography}{10}
\providecommand{\url}[1]{#1}
\csname url@samestyle\endcsname
\providecommand{\newblock}{\relax}
\providecommand{\bibinfo}[2]{#2}
\providecommand{\BIBentrySTDinterwordspacing}{\spaceskip=0pt\relax}
\providecommand{\BIBentryALTinterwordstretchfactor}{4}
\providecommand{\BIBentryALTinterwordspacing}{\spaceskip=\fontdimen2\font plus
\BIBentryALTinterwordstretchfactor\fontdimen3\font minus
  \fontdimen4\font\relax}
\providecommand{\BIBforeignlanguage}[2]{{%
\expandafter\ifx\csname l@#1\endcsname\relax
\typeout{** WARNING: IEEEtran.bst: No hyphenation pattern has been}%
\typeout{** loaded for the language `#1'. Using the pattern for}%
\typeout{** the default language instead.}%
\else
\language=\csname l@#1\endcsname
\fi
#2}}
\providecommand{\BIBdecl}{\relax}
\BIBdecl

\bibitem{zhou2003learning}
D.~Zhou, O.~Bousquet, T.~Lal, J.~Weston, and B.~Sch{\"o}lkopf, ``Learning with
  local and global consistency,'' \emph{Advances in neural information
  processing systems}, vol.~16, pp. 321--328, 2003.

\bibitem{chapelle2006label}
O.~Chapelle, B.~Sch{\"o}lkopf, and A.~Zien, ``Label propagation and quadratic
  criterion,'' 2006.

\bibitem{ravi2016large}
S.~Ravi and Q.~Diao, ``Large scale distributed semi-supervised learning using
  streaming approximation,'' in \emph{Artificial Intelligence and Statistics},
  2016, pp. 519--528.

\bibitem{lecun2015deep}
Y.~LeCun, Y.~Bengio, and G.~Hinton, ``Deep learning,'' \emph{nature}, vol. 521,
  no. 7553, pp. 436--444, 2015.

\bibitem{kipf2016semi}
T.~N. Kipf and M.~Welling, ``Semi-supervised classification with graph
  convolutional networks,'' \emph{arXiv preprint arXiv:1609.02907}, 2016.

\bibitem{li2019label}
Q.~Li, X.-M. Wu, H.~Liu, X.~Zhang, and Z.~Guan, ``Label efficient
  semi-supervised learning via graph filtering,'' in \emph{Proceedings of the
  IEEE Conference on Computer Vision and Pattern Recognition}, 2019, pp.
  9582--9591.

\bibitem{ma2019flexible}
J.~Ma, W.~Tang, J.~Zhu, and Q.~Mei, ``A flexible generative framework for
  graph-based semi-supervised learning,'' in \emph{Advances in Neural
  Information Processing Systems}, 2019, pp. 3281--3290.

\bibitem{qu2019gmnn}
M.~Qu, Y.~Bengio, and J.~Tang, ``Gmnn: Graph markov neural networks,''
  \emph{arXiv preprint arXiv:1905.06214}, 2019.

\bibitem{zhang2019bayesian}
Y.~Zhang, S.~Pal, M.~Coates, and D.~Ustebay, ``Bayesian graph convolutional
  neural networks for semi-supervised classification,'' in \emph{Proceedings of
  the AAAI Conference on Artificial Intelligence}, vol.~33, 2019, pp.
  5829--5836.

\bibitem{chapelle2009semi}
O.~Chapelle, B.~Scholkopf, and A.~Zien, ``Semi-supervised learning (chapelle,
  o. et al., eds.; 2006)[book reviews],'' \emph{IEEE Transactions on Neural
  Networks}, vol.~20, no.~3, pp. 542--542, 2009.

\bibitem{chiang2019cluster}
W.-L. Chiang, X.~Liu, S.~Si, Y.~Li, S.~Bengio, and C.-J. Hsieh, ``Cluster-gcn:
  An efficient algorithm for training deep and large graph convolutional
  networks,'' in \emph{Proceedings of the 25th ACM SIGKDD International
  Conference on Knowledge Discovery \& Data Mining}, 2019, pp. 257--266.

\bibitem{wu2019simplifying}
F.~Wu, T.~Zhang, A.~H.~d. Souza~Jr, C.~Fifty, T.~Yu, and K.~Q. Weinberger,
  ``Simplifying graph convolutional networks,'' \emph{arXiv preprint
  arXiv:1902.07153}, 2019.

\bibitem{zhu2009introduction}
X.~Zhu and A.~B. Goldberg, ``Introduction to semi-supervised learning,''
  \emph{Synthesis lectures on artificial intelligence and machine learning},
  vol.~3, no.~1, pp. 1--130, 2009.

\bibitem{ou2016asymmetric}
M.~Ou, P.~Cui, J.~Pei, Z.~Zhang, and W.~Zhu, ``Asymmetric transitivity
  preserving graph embedding,'' in \emph{Proceedings of the 22nd ACM SIGKDD
  international conference on Knowledge discovery and data mining}, 2016, pp.
  1105--1114.

\bibitem{kingma2013auto}
D.~P. Kingma and M.~Welling, ``Auto-encoding variational bayes,'' \emph{arXiv
  preprint arXiv:1312.6114}, 2013.

\bibitem{goodfellow2014generative}
I.~Goodfellow, J.~Pouget-Abadie, M.~Mirza, B.~Xu, D.~Warde-Farley, S.~Ozair,
  A.~Courville, and Y.~Bengio, ``Generative adversarial nets,'' in
  \emph{Advances in neural information processing systems}, 2014, pp.
  2672--2680.

\bibitem{zhang2018link}
M.~Zhang and Y.~Chen, ``Link prediction based on graph neural networks,'' in
  \emph{Advances in Neural Information Processing Systems}, 2018, pp.
  5165--5175.

\bibitem{hamilton2017inductive}
W.~Hamilton, Z.~Ying, and J.~Leskovec, ``Inductive representation learning on
  large graphs,'' in \emph{Advances in neural information processing systems},
  2017, pp. 1024--1034.

\bibitem{hotelling1933analysis}
H.~Hotelling, ``Analysis of a complex of statistical variables into principal
  components.'' \emph{Journal of educational psychology}, vol.~24, no.~6, p.
  417, 1933.

\bibitem{paszke2017automatic}
A.~Paszke, S.~Gross, S.~Chintala, G.~Chanan, E.~Yang, Z.~DeVito, Z.~Lin,
  A.~Desmaison, L.~Antiga, and A.~Lerer, ``Automatic differentiation in
  pytorch,'' 2017.

\bibitem{10.5555/3041838.3041953}
X.~Zhu, Z.~Ghahramani, and J.~Lafferty, ``Semi-supervised learning using
  gaussian fields and harmonic functions,'' in \emph{Proceedings of the
  Twentieth International Conference on International Conference on Machine
  Learning}, ser. ICML'03.\hskip 1em plus 0.5em minus 0.4em\relax AAAI Press,
  2003, p. 912–919.

\bibitem{perozzi2014deepwalk}
B.~Perozzi, R.~Al-Rfou, and S.~Skiena, ``Deepwalk: Online learning of social
  representations,'' in \emph{Proceedings of the 20th ACM SIGKDD international
  conference on Knowledge discovery and data mining}, 2014, pp. 701--710.

\bibitem{Tang_2015}
\BIBentryALTinterwordspacing
J.~Tang, M.~Qu, M.~Wang, M.~Zhang, J.~Yan, and Q.~Mei, ``Line,''
  \emph{Proceedings of the 24th International Conference on World Wide Web -
  WWW ’15}, 2015. [Online]. Available:
  \url{http://dx.doi.org/10.1145/2736277.2741093}
\BIBentrySTDinterwordspacing

\bibitem{48921}
\BIBentryALTinterwordspacing
P.~Veličković, W.~Fedus, W.~L. Hamilton, P.~Liò, Y.~Bengio, and R.~D. Hjelm,
  ``Deep graph infomax,'' 2019. [Online]. Available:
  \url{https://openreview.net/forum?id=rklz9iAcKQ}
\BIBentrySTDinterwordspacing

\bibitem{bojchevski2017deep}
A.~Bojchevski and S.~G{\"u}nnemann, ``Deep gaussian embedding of graphs:
  Unsupervised inductive learning via ranking,'' \emph{arXiv preprint
  arXiv:1707.03815}, 2017.

\bibitem{velickovic2018graph}
\BIBentryALTinterwordspacing
P.~Veli{\v{c}}kovi{\'{c}}, G.~Cucurull, A.~Casanova, A.~Romero, P.~Li{\`{o}},
  and Y.~Bengio, ``{Graph Attention Networks},'' \emph{International Conference
  on Learning Representations}, 2018. [Online]. Available:
  \url{https://openreview.net/forum?id=rJXMpikCZ}
\BIBentrySTDinterwordspacing

\bibitem{li2018deeper}
Q.~Li, Z.~Han, and X.-M. Wu, ``Deeper insights into graph convolutional
  networks for semi-supervised learning,'' \emph{arXiv preprint
  arXiv:1801.07606}, 2018.

\bibitem{sun2019multi}
K.~Sun, Z.~Zhu, and Z.~Lin, ``Multi-stage self-supervised learning for graph
  convolutional networks,'' \emph{arXiv preprint arXiv:1902.11038}, 2019.

\bibitem{chen2018fastgcn}
J.~Chen, T.~Ma, and C.~Xiao, ``Fastgcn: fast learning with graph convolutional
  networks via importance sampling,'' \emph{arXiv preprint arXiv:1801.10247},
  2018.

\bibitem{you2020l2}
Y.~You, T.~Chen, Z.~Wang, and Y.~Shen, ``L2-gcn: Layer-wise and learned
  efficient training of graph convolutional networks,'' in \emph{Proceedings of
  the IEEE/CVF Conference on Computer Vision and Pattern Recognition}, 2020,
  pp. 2127--2135.

\bibitem{he2016deep}
K.~He, X.~Zhang, S.~Ren, and J.~Sun, ``Deep residual learning for image
  recognition,'' in \emph{Proceedings of the IEEE conference on computer vision
  and pattern recognition}, 2016, pp. 770--778.

\bibitem{kingma2014semi}
D.~P. Kingma, S.~Mohamed, D.~Jimenez~Rezende, and M.~Welling, ``Semi-supervised
  learning with deep generative models,'' \emph{Advances in neural information
  processing systems}, vol.~27, pp. 3581--3589, 2014.

\bibitem{bennett1999semi}
K.~P. Bennett and A.~Demiriz, ``Semi-supervised support vector machines,'' in
  \emph{Advances in Neural Information processing systems}, 1999, pp. 368--374.

\bibitem{grandvalet2005semi}
Y.~Grandvalet and Y.~Bengio, ``Semi-supervised learning by entropy
  minimization,'' in \emph{Advances in neural information processing systems},
  2005, pp. 529--536.

\bibitem{belkin2006manifold}
M.~Belkin, P.~Niyogi, and V.~Sindhwani, ``Manifold regularization: A geometric
  framework for learning from labeled and unlabeled examples,'' \emph{Journal
  of machine learning research}, vol.~7, no. Nov, pp. 2399--2434, 2006.

\bibitem{liang2018lightweight}
D.-M. Liang and Y.-F. Li, ``Lightweight label propagation for large-scale
  network data.'' in \emph{IJCAI}, 2018, pp. 3421--3427.

\bibitem{iscen2019label}
A.~Iscen, G.~Tolias, Y.~Avrithis, and O.~Chum, ``Label propagation for deep
  semi-supervised learning,'' in \emph{Proceedings of the IEEE conference on
  computer vision and pattern recognition}, 2019, pp. 5070--5079.

\bibitem{gong2015deformed}
C.~Gong, T.~Liu, D.~Tao, K.~Fu, E.~Tu, and J.~Yang, ``Deformed graph laplacian
  for semisupervised learning,'' \emph{IEEE transactions on neural networks and
  learning systems}, vol.~26, no.~10, pp. 2261--2274, 2015.

\bibitem{zhu2002towards}
X.~Zhu and Z.~Ghahramani, ``Towards semi-supervised classification with markov
  random fields,'' 2002.

\bibitem{belkin2004semi}
M.~Belkin and P.~Niyogi, ``Semi-supervised learning on riemannian manifolds,''
  \emph{Machine learning}, vol.~56, no. 1-3, pp. 209--239, 2004.

\bibitem{chapelle2002cluster}
O.~Chapelle, J.~Weston, and B.~Sch{\"o}lkopf, ``Cluster kernels for
  semi-supervised learning,'' \emph{Advances in neural information processing
  systems}, vol.~15, pp. 601--608, 2002.

\bibitem{grover2016node2vec}
A.~Grover and J.~Leskovec, ``node2vec: Scalable feature learning for
  networks,'' in \emph{Proceedings of the 22nd ACM SIGKDD international
  conference on Knowledge discovery and data mining}, 2016, pp. 855--864.

\bibitem{sen2008collective}
P.~Sen, G.~Namata, M.~Bilgic, L.~Getoor, B.~Galligher, and T.~Eliassi-Rad,
  ``Collective classification in network data,'' \emph{AI magazine}, vol.~29,
  no.~3, pp. 93--93, 2008.

\bibitem{weston2012deep}
J.~Weston, F.~Ratle, H.~Mobahi, and R.~Collobert, ``Deep learning via
  semi-supervised embedding,'' in \emph{Neural networks: Tricks of the
  trade}.\hskip 1em plus 0.5em minus 0.4em\relax Springer, 2012, pp. 639--655.

\bibitem{yang2016revisiting}
Z.~Yang, W.~Cohen, and R.~Salakhudinov, ``Revisiting semi-supervised learning
  with graph embeddings,'' in \emph{International conference on machine
  learning}.\hskip 1em plus 0.5em minus 0.4em\relax PMLR, 2016, pp. 40--48.

\bibitem{lecun1995convolutional}
Y.~LeCun, Y.~Bengio \emph{et~al.}, ``Convolutional networks for images, speech,
  and time series,'' \emph{The handbook of brain theory and neural networks},
  vol. 3361, no.~10, p. 1995, 1995.

\bibitem{henaff2015deep}
M.~Henaff, J.~Bruna, and Y.~LeCun, ``Deep convolutional networks on
  graph-structured data,'' \emph{arXiv preprint arXiv:1506.05163}, 2015.

\bibitem{bruna2013spectral}
J.~Bruna, W.~Zaremba, A.~Szlam, and Y.~LeCun, ``Spectral networks and locally
  connected networks on graphs,'' \emph{arXiv preprint arXiv:1312.6203}, 2013.

\bibitem{atwood2016diffusion}
J.~Atwood and D.~Towsley, ``Diffusion-convolutional neural networks,'' in
  \emph{Advances in neural information processing systems}, 2016, pp.
  1993--2001.

\bibitem{duvenaud2015convolutional}
D.~K. Duvenaud, D.~Maclaurin, J.~Iparraguirre, R.~Bombarell, T.~Hirzel,
  A.~Aspuru-Guzik, and R.~P. Adams, ``Convolutional networks on graphs for
  learning molecular fingerprints,'' in \emph{Advances in neural information
  processing systems}, 2015, pp. 2224--2232.

\bibitem{shuman2013emerging}
D.~I. Shuman, S.~K. Narang, P.~Frossard, A.~Ortega, and P.~Vandergheynst, ``The
  emerging field of signal processing on graphs: Extending high-dimensional
  data analysis to networks and other irregular domains,'' \emph{IEEE signal
  processing magazine}, vol.~30, no.~3, pp. 83--98, 2013.

\bibitem{defferrard2016convolutional}
M.~Defferrard, X.~Bresson, and P.~Vandergheynst, ``Convolutional neural
  networks on graphs with fast localized spectral filtering,'' in
  \emph{Advances in neural information processing systems}, 2016, pp.
  3844--3852.

\bibitem{xu2018powerful}
K.~Xu, W.~Hu, J.~Leskovec, and S.~Jegelka, ``How powerful are graph neural
  networks?'' \emph{arXiv preprint arXiv:1810.00826}, 2018.

\bibitem{spinelli2020adaptive}
I.~Spinelli, S.~Scardapane, and A.~Uncini, ``Adaptive propagation graph
  convolutional network,'' \emph{arXiv preprint arXiv:2002.10306}, 2020.

\bibitem{wu2020comprehensive}
Z.~Wu, S.~Pan, F.~Chen, G.~Long, C.~Zhang, and S.~Y. Philip, ``A comprehensive
  survey on graph neural networks,'' \emph{IEEE Transactions on Neural Networks
  and Learning Systems}, 2020.

\bibitem{nt2019revisiting}
H.~NT and T.~Maehara, ``Revisiting graph neural networks: All we have is
  low-pass filters,'' \emph{arXiv preprint arXiv:1905.09550}, 2019.

\bibitem{qu2020gmnn}
M.~Qu, Y.~Bengio, and J.~Tang, ``Gmnn: Graph markov neural networks,'' 2020.

\bibitem{you2020does}
Y.~You, T.~Chen, Z.~Wang, and Y.~Shen, ``When does self-supervision help graph
  convolutional networks?'' \emph{arXiv preprint arXiv:2006.09136}, 2020.

\bibitem{lee2013pseudo}
D.-H. Lee, ``Pseudo-label: The simple and efficient semi-supervised learning
  method for deep neural networks,'' in \emph{Workshop on challenges in
  representation learning, ICML}, vol.~3, no.~2, 2013.

\bibitem{li2020seal}
Y.~Li, J.~Yin, and L.~Chen, ``Seal: Semisupervised adversarial active learning
  on attributed graphs,'' \emph{IEEE Transactions on Neural Networks and
  Learning Systems}, 2020.

\bibitem{yu2018active}
H.~Yu, X.~Yang, S.~Zheng, and C.~Sun, ``Active learning from imbalanced data: A
  solution of online weighted extreme learning machine,'' \emph{IEEE
  Transactions on Neural Networks and Learning Systems}, vol.~30, no.~4, pp.
  1088--1103, 2018.

\bibitem{zeng2019graphsaint}
H.~Zeng, H.~Zhou, A.~Srivastava, R.~Kannan, and V.~Prasanna, ``Graphsaint:
  Graph sampling based inductive learning method,'' \emph{arXiv preprint
  arXiv:1907.04931}, 2019.

\bibitem{tarvainen2017mean}
A.~Tarvainen and H.~Valpola, ``Mean teachers are better role models:
  Weight-averaged consistency targets improve semi-supervised deep learning
  results,'' in \emph{Advances in neural information processing systems}, 2017,
  pp. 1195--1204.

\bibitem{berthelot2019mixmatch}
D.~Berthelot, N.~Carlini, I.~Goodfellow, N.~Papernot, A.~Oliver, and C.~A.
  Raffel, ``Mixmatch: A holistic approach to semi-supervised learning,'' in
  \emph{Advances in Neural Information Processing Systems}, 2019, pp.
  5049--5059.

\bibitem{8417973}
T.~{Miyato}, S.~{Maeda}, M.~{Koyama}, and S.~{Ishii}, ``Virtual adversarial
  training: A regularization method for supervised and semi-supervised
  learning,'' \emph{IEEE Transactions on Pattern Analysis and Machine
  Intelligence}, vol.~41, no.~8, pp. 1979--1993, 2019.

\bibitem{hinton2015distilling}
G.~Hinton, O.~Vinyals, and J.~Dean, ``Distilling the knowledge in a neural
  network,'' 2015.

\end{thebibliography}

\end{document}